\documentclass{article}

\usepackage[utf8]{inputenc} 
\usepackage[T1]{fontenc}    
\usepackage{hyperref}       
\usepackage{url}            
\usepackage{booktabs}       
\usepackage{amsfonts}       
\usepackage{nicefrac}       
\usepackage{microtype}      
\usepackage{xcolor}         
\usepackage{graphicx}
\usepackage{subcaption}
\usepackage{threeparttable}
\usepackage{tabularx}
\usepackage{booktabs}
\usepackage{amsmath}
\usepackage{amssymb}
\usepackage{mathtools}
\usepackage{amsthm}
\usepackage{algorithm}
\usepackage{algpseudocode}
\usepackage[table]{xcolor}
\usepackage{wrapfig}
\usepackage{multirow}

\theoremstyle{plain}
\newtheorem{theorem}{Theorem}[section]

\theoremstyle{definition}

\theoremstyle{remark}
\newtheorem{remark}[theorem]{Remark}

\title{ProbRes: Volatility Learning for Probabilistic Time-Series Forecasting}

\author{Tingting Wang$^\dagger$, Yunyi Zhang$^\ddagger$, Benyou Wang$^\star$\\
School of Data Science, The Chinese University of Hong Kong, Shenzhen\\
$\dagger:$  wangtingting@cuhk.edu.cn\\
$\ddagger:$ zhangyunyi@cuhk.edu.cn\\
$\star:$ wangbenyou@cuhk.edu.cn
}

\date{}

\begin{document}

\maketitle

\begin{abstract}
  Probabilistic time series forecasting has attracted increasing attention in financial applications due to the need to quantify risk and uncertainty in future observations. We propose ProbRes, a post-hoc probabilistic calibration method that explicitly learns and incorporates volatility dynamics into probabilistic forecasting, enabling effective handling of heteroskedastic data. During training, ProbRes employs two architecture-agnostic modules to separately model the conditional mean and conditional volatility. At the inference stage, it generates predictive distributions by resampling normalized residuals. ProbRes is applicable to both univariate and multivariate time series and remains robust under a wide range of error distributions, including non-Gaussian innovations with conditional heteroskedasticity. Theoretical results demonstrate ProbRes's validity and  experiments on both synthetic and real-world datasets show that ProbRes accurately captures predictive distributions and produces well-calibrated prediction intervals.
\end{abstract}

\section{Introduction}
\label{introduction}
Time series data arise naturally in financial applications, including asset pricing, risk management, and economic indicator analysis. 
In these domains, probabilistic time series forecasting \cite{annurev}, which is the task of predicting the probability distribution of future observations, plays a crucial role in risk assessment and decision-making, as discussed in \cite{Luo_Zhang_Xu_Wang_2018, Nguyen_Quanz_2021, https://doi.org/10.1111/jtsa.12739, pacchiardi2024probabilistic, zheng2025koonpro}. Compared to point forecasts, probabilistic forecasts quantify uncertainty and capture variability in the data, providing richer information for downstream tasks such as portfolio optimization, trading strategy decisions, or extreme event prediction.

Real-world financial data exhibit strong time-varying volatility. For example,   asset returns commonly display volatility clustering (volatility
becomes high for certain time periods and low for other periods) and their conditional mean may depend on volatility, reflecting a mean–volatility interaction \cite{tsay2005analysis}. These features motivate explicitly modeling conditional volatility in probabilistic forecasting. To this end, we propose ProbRes, a post-hoc probabilistic calibration method that decomposes observed time series into three components: the conditional mean, the conditional volatility (i.e. conditional standard deviation), and the distribution of normalized residuals. These components respectively characterize the expected value, variability around the mean, and higher-order distributional characteristics (e.g., skewness and tail behavior) of future observations. Building on this decomposition, ProbRes enables practitioners to effectively model these three components and leverage them for probabilistic forecasting.

Compared to existing probabilistic forecasting methods---including bootstrap methods \cite{PAN20161, https://doi.org/10.1111/jtsa.12739}, quantile-based wrappers \cite{pmlr-v151-park22a, yu2026nonparametric}, and end-to-end models \cite{NEURIPS2023_5a1a10c2, ansari2024chronos, zheng2025koonpro}---ProbRes explicitly models the functional dependence of conditional volatility on past observations, in the spirit of classical volatility models \cite{6ab571e5-c8f0-3fcd-9005-ed6f2adc76d7, BOLLERSLEV1986307}, rather than treating it as a nuisance component \cite{FAN2019261}. This makes ProbRes particularly suitable for capturing volatility dynamics and reproducing volatility clustering.

Modeling conditional volatility is of direct practical importance and affects probabilistic forecasting performance. In asset pricing and risk management, volatility is a key indicator of downside risk and is directly related to widely used measures such as value-at-risk, as well as the pricing of derivatives and hedging strategies \cite{abken1996options, tsay2005analysis, DESILVA2017365}. Table \ref{table.variance_of_ARCH_model} illustrates the second point using an ARCH(1) process: although the unconditional variance of $X_t$ is 2, its conditional variance depends on $X_{t-1}.$ Since forecasting relies on the conditional distribution given past data, ignoring conditional heteroskedasticity---e.g., using a constant variance of 2 instead of $1 + 0.5 x^2$ --—can lead to loss in  predictive performance.


\begin{table}[t]
\centering
\caption{Unconditional and conditional variances of the ARCH(1) model 
$X_t = \epsilon_t \sqrt{1 + 0.5X_{t-1}^2}$, where $\epsilon_t \sim i.i.d.\  \mathcal{N}(0,1)$.}
\label{table.variance_of_ARCH_model}
\begin{tabular}{cc}
\toprule
Variance & Conditional variance at $X_{t-1}=x$ \\
\midrule
2 & $1 + 0.5x^2$ \\
\bottomrule
\end{tabular}
\end{table}

The architecture of ProbRes follows a three-stage procedure. In the first stage, we train a mean (point) forecasting model---such as those in \cite{Zeng_Chen_Zhang_Xu_2023, lin2024cyclenet}---to predict the conditional mean of future observations. We then compute the fitted residuals as the difference between the observed values and the predicted means. In the second stage, we take the logarithm of the squared fitted residuals and train a mean forecasting model on these transformed residuals, which capture conditional volatility information. Finally, we normalize the original residuals by dividing them by the predicted volatilities and apply bootstrap algorithms---such as those in \cite{MR515681}---to estimate the distribution of future observations. As demonstrated in \cite{10.1214/aos/1176350142, 8487c678-a212-3cd7-a6ca-cb4eefa4828e, NIPS2014_4e382cb4, PAN20161} and further established by Theorem \ref{theorem.consistent} below, a well-designed bootstrap algorithm can approximate the underlying probability distribution of future time series without imposing restrictive distributional assumptions. Overall, this architecture and its theoretical properties enable ProbRes to generalize across diverse time series settings and adapt to a wide range of mean forecasting models with little additional training.

In addition to capturing conditional volatility information and addressing the conditional heteroskedastic issue, ProbRes offers some practical advantages. First, it provides flexibility in the choice of conditional mean and volatility models. As shown in Section \ref{section.further_discussion}, by applying a logarithmic transformation to the squared residuals, ProbRes requires only two mean-forecasting models to perform probabilistic forecasting. This allows models originally designed for mean forecasting to be adapted for probabilistic forecasting. Second, as established in Theorem \ref{theorem.consistent}, ProbRes incorporates spatial dependence by resampling residual vectors, making it adaptable to multivariate time series settings.

We summarize the advantages of the proposed method as follows.
\begin{itemize}

    \item {\bf Direct modeling of volatility dynamics and robustness to multivariate time series:} ProbRes learns the functional relationship between conditional volatility and past errors, enabling it to capture phenomena such as volatility clustering.
    Moreover, its residual resampling mechanism allows it to capture spatial dependencies.

    \item {\bf Good  interpretability and theoretical justification: } 
    The decomposition allows practitioners to identify and interpret the contributions of different components to the probabilistic forecasts, enhancing its practical utility. Importantly, the validity of ProbRes stems from its ability to emulate the underlying data-generating process of the time series, rather than relying on a black-box model. Furthermore, under certain conditions, the residual resampling mechanism is guaranteed to recover the true distribution of the innovations.

   \item {\bf Few additional training costs and flexible mean/volatility modeling:}
As a post-hoc probabilistic calibration method, ProbRes does not require modifying the underlying model architecture or loss function (e.g., adopting pinball loss) for mean forecasting. Practitioners can directly plug in existing mean-forecasting models to obtain probabilistic forecasts. Empirical results further show that a lightweight multilayer perceptron suffices to capture volatility dynamics, introducing negligible additional training overhead.

\end{itemize}

\section{Related work and motivations}

This work is closely related to two areas of literature: post-hoc probabilistic calibration and wrapper methods, including bootstrap, conformal prediction, and quantile-based forecasting; and end-to-end probabilistic time series forecasting approaches, such as diffusion-based methods. We also introduce the motivation for explicitly modeling volatility dynamics.

\textbf{Post-hoc probabilistic calibration.}    Bootstrap algorithms and conformal prediction methods such as those in \cite{NIPS2010_13f3cf8c, NEURIPS2021_58b7483b, pmlr-v235-wang24ah, yu2024metamath, romano2019conformalized, 10.5555/3618408.3620021, angelopoulos2023conformal}, provide post-hoc calibration techniques for generating probabilistic forecasts without requiring additional training or modifications to model architectures, aligning well with our setting. We further refer to \cite{NEURIPS2023_aef75887, chen2024conformalized} for recent developments in conformal prediction for time series.

An alternative approach to obtaining probabilistic forecasts---also without modifying model architectures---is to change loss functions, such as adopting the pinball loss in quantile regression \cite{JMLR:v7:takeuchi06a, pmlr-v151-park22a, 10.5555/3692070.3694179, yu2026nonparametric}. However, as noted in \cite{FAN2019261}, classical quantile regression typically treats heteroskedasticity as a nuisance rather than explicitly modeling volatility dynamics. In contrast, our work directly models volatility, providing a good alternative to quantile regression for capturing heteroskedastic behavior.

\textbf{End-to-end probabilistic forecasting.} Diffusion models, including \cite{pmlr-v139-rasul21a, rasul2021multivariate, NEURIPS2022_91a85f3f, NEURIPS2023_5a1a10c2, chen2024recurrent, pmlr-v235-chen24n, 10.5555/3540261.3542161, li2024transformermodulated, zheng2025koonpro, kollovieh2025flow} and their variants \cite{NEURIPS2020_4c5bcfec, song2020score}, are trained to learn a transformation between the data distribution and a tractable reference distribution. Once this transformation is learned, samples drawn from the reference distribution can be mapped back to the data space to generate predictive samples, enabling end-to-end probabilistic forecasting. We further refer to \cite{Gao_Cao_Chen_2025} and \cite{NEURIPS2018_5cf68969, 10.5555/3367243.3367442} for approaches that incorporate sequential and state-space structures into diffusion models, and to \cite{pmlr-v139-rangapuram21a, 10.1109/TKDE.2023.3319672, ansari2024chronos} for other deep learning-based probabilistic forecasting methods. Another line of work addresses probabilistic forecasting by modifying training objectives or learning procedures \cite{NEURIPS2020_2f2b2656, rasul2021multivariate, NEURIPS2021_32b12730, NEURIPS2023_6171c9e6, ansari2024chronos}.

While effective in practice, these approaches typically rely on specialized model architectures and complex end-to-end training pipelines. Moreover, they do not explicitly model the functional relationship between future volatility and past observations.


\textbf{Motivation for studying volatility dynamics. }  
Our method learns the functional relationship between volatility dynamics and past observations and residuals, which is motivated by classical financial time series models such as ARCH and GARCH \cite{engle1982autoregressive, BOLLERSLEV1986307} and their extensions \cite{tsay2005analysis}. Beyond improving forecasting performance, explicitly modeling volatility helps capture financial phenomena such as volatility clustering. The study of conditional volatility is
motivated by its central role in financial applications, including option pricing, trading, and asset pricing \cite{yang2025hyperiv, https://doi.org/10.1111/j.1540-6261.2008.01352.x, https://doi.org/10.1111/jofi.12110}. A growing body of machine learning methods \cite{NEURIPS2020_858e4770, doi:10.1137/21M1443546, wiedemann2025operator} has also been developed for nowcasting implied volatility surfaces. We further refer to \cite{ZHANG2025377, 10.1145/3447548.3467115} for related work on volatility surface prediction.

\section{Resampling assisted probabilistic forecasting (ProbRes)}
\label{section.setting}
Suppose we observe a time series $\mathbf{x}_{1:T},$ where $\mathbf{x}_{t}\in\mathbf{R}^d,$ and we denote the subscript $\ t = 1,\cdots, T$ as the time stamp. The objective of probabilistic forecasting is to forecast the distributions of future observations $\mathbf{x}_{(T+1):(T + J)}$ with $J$ the prediction length. 
This setup has been discussed in the literature such as \cite{SALINAS20201181, kollovieh2025flow}. Our work aims to explicitly learn the conditional mean, volatility, and the normalized residuals' distribution information, and leverage these three aspects of information in forecasting. 
To achieve this goal, we incorporate a resampling step into the forecasting algorithm \ref{alg.pred}. Resampling has been well employed in the literature such as \cite{PAN20161}, \cite{10.1145/3711709}, and \cite{ZHANG2025112611} in forecasting. However, to our knowledge, they did not account for the existence of conditional heteroskedasticity (dependence of future variance on past observations). Our work addresses this issue  by normalizing the fitted residuals with the forecasted conditional volatility prior to resampling.

Our framework consists of two stages: a training stage, where models are trained to learn the conditional mean and volatility,  and the normalized fitted residuals are stored; and a forecasting stage, where the trained models are used to generate probabilistic forecasts.

\subsection{Training stage}

Our work is motivated by a two-stage conditionally heteroskedastic vector autoregressive model of the form

\begin{equation}
\left\{
\begin{aligned}
& \mathbf{x}_t = F(\mathbf{x}_{t-1},\ldots,\mathbf{x}_{t-q}) + \boldsymbol{\zeta}_t, \\
& \boldsymbol{\zeta}_t = G(\boldsymbol{\zeta}_{t-1},\ldots,\boldsymbol{\zeta}_{t-s})\boldsymbol{\eta}_t .
\end{aligned}
\right.
\label{eq.heterogeneous_autoregressive}
\end{equation}

where $G(\boldsymbol{\zeta}_{t-1},\cdots,\boldsymbol{\zeta}_{t-s})$ is a $d\times d$ diagonal matrix with diagonal elements $G_1(\boldsymbol{\zeta}_{t-1},\cdots,\boldsymbol{\zeta}_{t-s}), \cdots$ $G_d(\boldsymbol{\zeta}_{t-1},\cdots,\boldsymbol{\zeta}_{t-s}).$  
The conditional mean function 
$F:\mathbf{R}^{d\times q}\to\mathbf{R}^d$ and the conditional volatility functions $G_i: \mathbf{R}^{d\times s}\to [0,\infty)$ are functions to learn, and $\boldsymbol{\eta}_{t}$ are independent of past observations $\mathbf{x}_{-t}$ and $\boldsymbol{\zeta}_{-t}$,  $\mathbf{E}\left[\boldsymbol{\eta}^{(t)}\right] = 0.$ We further assume that $\boldsymbol{\eta}^{(t)}$ have identical distribution.

The functions $F$ and $G_i, i = 1,\cdots,d$ respectively control the conditional mean and conditional volatility of time series data.  Furthermore, such a model offers a good property that the residual terms $\boldsymbol{\zeta}_t$ do not incur bias to the conditional mean $F,$ which motivates the two-stage training procedure as in Algorithm \ref{alg.train}.  We prove this property in Section \ref{section.theoretical_analysis}.


\begin{algorithm}[H]
\caption{Training a heteroskedastic vector autoregressive model}
\label{alg.train}
\begin{algorithmic}[1]

\Require Time series data $\{\mathbf{x}_t: t = 1,\cdots, T\}$, lag $q$ for the conditional mean model, and lag $s$ for the conditional volatility model.

\State Train the conditional mean model $\widehat{F}$ and derive the fitted residuals
\[
\widehat{\boldsymbol{\zeta}}_t = \mathbf{x}_t - \widehat{F}(\mathbf{x}_{t-q},\cdots,\mathbf{x}_{t-1}),
\]
for $t = q + 1, \cdots, T$.

\State Train the conditional volatility model $\widehat{G}$ with the fitted residuals $\widehat{\boldsymbol{\zeta}}_t$, $t = q+1,\cdots,T$. Then derive the normalized fitted residuals
\begin{equation}
\widehat{\boldsymbol{\eta}}_t = \widehat{G}^{-1}\left(\widehat{\boldsymbol{\zeta}}_{t-s},\cdots,\widehat{\boldsymbol{\zeta}}_{t-1}\right)\widehat{\boldsymbol{\zeta}}_t,
\label{eq.normalize}
\end{equation}
for $t = q+s+1,\cdots, T$.

\end{algorithmic}
\end{algorithm}

\begin{remark}
\label{remark.train_G}
    Practitioners may resort to mean forecasting methods, such as \cite{lin2024cyclenet}, to establish the model $\widehat{F}$ for the conditional mean function $F$ in \eqref{eq.heterogeneous_autoregressive}. Learning $G,$ on the other hand, is not straightforward. After calculating $\widehat{\boldsymbol{\zeta}}_t,$ this manuscript performs the transformation $\widehat{\boldsymbol{\iota}}_t = R(\widehat{\boldsymbol{\zeta}}_t)$   for $t = q + 1,\cdots, T,$ where $R:\mathbf{R}^d\to\mathbf{R}^d$ is a function of the form:
    \begin{equation}
    R(\mathbf{x}) = (\log(\mathbf{x}_1^2), \log(\mathbf{x}_2^2),\cdots, \log(\mathbf{x}_d^2))^\top.
    \label{eq.def_R}
    \end{equation}
    We then use mean forecasting methods (e.g., those in \cite{lin2024cyclenet}) to learn $U_i = \log(G_i).$ We  demonstrate in Section \ref{section.further_discussion} that, although taking logarithmic transformations incurs a constant bias when learning  $\log(G_i),$  the constant bias will be self-eliminated during the normalization step \eqref{eq.normalize} of Algorithm \ref{alg.train} and the sampling step \eqref{eq.sample} of  Algorithm \ref{alg.pred}. Consequently, the bias 
introduced during the training stage does not affect forecasting.
\end{remark}

    The motivation of model \eqref{eq.heterogeneous_autoregressive} originates from the ARMA-GARCH model, like those in \cite{Ling_McAleer_2003}, which adopted linear models for both $F$ and $G.$ The conditional heteroskedasticity considered in this manuscript associates volatility with past observations, and is different from \cite{ye2025nonstationarydiffusionprobabilistictime}, where the volatility was associated with exogenous features.

The flexibility of Algorithm \ref{alg.train} is reflected in its flexible selection of models used to learn $F$ and $G$---mean forecasting algorithms, such as those proposed in \cite{Zeng_Chen_Zhang_Xu_2023, zhang2023crossformer, lin2024cyclenet}, among others---can be employed to fulfill this purpose.

    Compared to classical diffusion models, which directly learn the probability distribution of time series, our framework imposes an autoregressive constraint on the model. We consider that this constraint is important because, as also noted in \cite{Gao_Cao_Chen_2025}, ignoring the sequential structure of time series can lead to a misalignment between diffusion mechanisms and the time series structure, and thereby affect predictive performance.

\subsection{Forecasting  stage}

The intuition behind Algorithm \ref{alg.pred} involves simulating the data generating process in \eqref{eq.heterogeneous_autoregressive}. If $\widehat{F}$ and $\widehat{G}$ closely approximate the true conditional mean $F$ and conditional volatilities $G,$ then Theorem \ref{theorem.consistent} in Section \ref{section.theoretical_analysis} guarantees that the distribution of the simulated normalized residuals $\boldsymbol{\eta}_{1:J}^*$ closely matches the distribution of the true normalized residuals $\boldsymbol{\eta}_{1:J}$. Furthermore, the generation of $\mathbf{x}_{(T + 1):(T+J)}^*$ follows the same autoregressive iteration as in  \eqref{eq.heterogeneous_autoregressive}. Therefore, under the assumption that \eqref{eq.heterogeneous_autoregressive} accurately characterizes the data generating process of $\mathbf{x}_{(T+1):(T+J)},$  since the estimated conditional mean $\widehat{F}$, conditional volatility $\widehat{G}$, the distribution of pseudo-normalized residuals $\boldsymbol{\eta}_{1:J}^*$, and the autoregressive iteration all provide good approximations to that of $\mathbf{x}_{(T+1):(T+J)}$, the distribution of the pseudo-samples $\mathbf{x}_{(T+1):(T+J)}^*$ should be close to that of the actual future observations $\mathbf{x}_{(T+1):(T+J)}$.

\begin{remark}
    Practitioners may resort to Remark \ref{remark.train_G} to learn $G.$ In such case, the value of $\widehat{G}(\widehat{\boldsymbol{\zeta}}_{T + j - s}^*,\cdots, \widehat{\boldsymbol{\zeta}}_{T + j- 1}^*)$ can be derived through applying the learned autoregressive model to $\widehat{\boldsymbol{\iota}}_{T+j-s}^*,\cdots, \widehat{\boldsymbol{\iota}}_{T+j-1}^*,$ where 
    $
    \widehat{\boldsymbol{\iota}}_{k}^* = R\left(\boldsymbol{\zeta}_{k}^{*}\right).
    $
    \label{remark.predict_A}
\end{remark}


\begin{algorithm}[H]
\caption{Inference Stage}
\label{alg.pred}
\begin{algorithmic}[1]
\Require Time series data $\mathbf{x}_{1:T}$, lag $q$ for conditional mean, lag $s$ for conditional volatility, prediction step $J$, resampling times  $B$.
\State Derive the functions $\widehat{F}$ and $\widehat{G}$, as well as the normalized fitted residuals $\widehat{\boldsymbol{\eta}}_t$ as in Algorithm~\ref{alg.train}.

\For{$b = 1$ to $B$}
    \State Sample $\boldsymbol{\eta}_{1:J}^*$ by drawing from $\widehat{\boldsymbol{\eta}}_{(q+s+1):T}$ with replacement.
    \State Generate pseudo-samples $\mathbf{x}_{(T+1):(T+J)}^*$ using:
    \begin{equation}
    \begin{aligned}
        \boldsymbol{\zeta}_{T+j}^* &= \widehat{G}(\widehat{\boldsymbol{\zeta}}_{T+j-s}^*, \cdots, \widehat{\boldsymbol{\zeta}}_{T+j-1}^*) \boldsymbol{\eta}_j^*, \\
        \mathbf{x}_{T+j}^* &= \widehat{F}(\mathbf{x}_{T+j-q}^*, \cdots, \mathbf{x}_{T+j-1}^*) + \boldsymbol{\zeta}_{T+j}^*,
    \end{aligned}
    \label{eq.sample}
    \end{equation}
    where $\mathbf{x}_{T+j-q}^* = \mathbf{x}_{T+j-q}$ if $q\geq j$ and $\boldsymbol{\zeta}_{T+j-s}^* = \widehat{\boldsymbol{\zeta}}_{T+j-s}$ if $s \geq j$.
\EndFor

\State \Return Forecasts $\mathbf{x}_{(T+1):(T+J)}^*$ for $b = 1,\cdots,B$.
\end{algorithmic}
\end{algorithm}

Figure \ref{fig:cdf_edf} illustrates why sampling with replacement from  $\widehat{\boldsymbol{\eta}}_{(q+ s + 1):\ T}$ can recover the underlying distribution of $\boldsymbol{\eta}_{(T+1):(T+J)}.$ According to \cite{MR515681}, sampling with replacement from $\widehat{\boldsymbol{\eta}}_{(q+ s + 1):\ T}$  is equivalent to generating observations from a distribution whose cumulative distribution function (CDF) $\widehat{P}(\cdot)$ is the empirical cumulative distribution function of $\widehat{\boldsymbol{\eta}}_{(q+ s + 1):\ T}$ of the form 
\begin{equation}
\widehat{P}(\mathbf{y}) = \frac{1}{T - q - s}\sum_{t = s + q + 1}^T
\mathbf{1}_\mathrm{\widehat{\boldsymbol{\eta}}_{t}\leq \mathbf{y}}
\label{eq.F_hat}
\end{equation}
where $\mathbf{1}_\mathrm{\widehat{\boldsymbol{\eta}}_{t}\leq \mathbf{y}}$ denotes for $\prod_{i = 1}^d \mathbf{1}_\mathrm{\widehat{\boldsymbol{\eta}}_{t,i}\leq \mathbf{y}_i},$ i.e., element-wise less than or equal to. As shown in Figure \ref{fig:cdf_edf},  the empirical CDF
closely approximates the underlying CDF of the data for moderate sample sizes. Therefore, the resampled $\boldsymbol{\eta}_{1:J}^*$ should be able to capture the distributional characteristics of $\boldsymbol{\eta}_{1:T}$. Note that resampling from $\widehat{P}(\cdot)$ also preserves element-wise dependence when $\boldsymbol{\eta}_t$ is a random vector, making our algorithm applicable to multivariate forecasting as well.


\section{Theoretical justification}
\label{section.theoretical_analysis}
The theoretical justification of ProbRes is divided into two parts. First, we provide illustrations on why Algorithm \ref{alg.train} is capable of learning $F$ and $G.$   After that, we summarize in Theorem \ref{theorem.consistent} that the distribution of the pseudo-normalized residuals $\boldsymbol{\eta}_{1:J}^{*}$ closely approximates that of the true normalized residuals $\boldsymbol{\eta}_{1:T}.$

\subsection{Further discussions on Section \ref{section.setting}}
\label{section.further_discussion}
To illustrate why the two-stage procedure in Algorithm \ref{alg.train} learns $F$ and $G,$ from the tower property of conditional expectation, 
\begin{align*}
\mathbf{E}\left[
\boldsymbol{\zeta}_t\mid \mathbf{x}_{(t-q): (t-1)}\right]
&= \mathbf{E}\left[\mathbf{E}\left[ G(\boldsymbol{\zeta}_{t-1},\cdots,\boldsymbol{\zeta}_{t-s})\boldsymbol{\eta}_t\mid \mathbf{x}_{(t-q): (t-1)},\ \boldsymbol{\zeta}_{(t-s):(t-1)}\right] \mid \mathbf{x}_{(t-q): (t-1)}\right]\\
& = \mathbf{E}\left[\left(G(\boldsymbol{\zeta}_{t-1},\cdots,\boldsymbol{\zeta}_{t-s})\mathbf{E}\boldsymbol{\eta}_t\right)\mid \mathbf{x}_{(t-q): (t-1)}\right] = 0.
\end{align*}
Therefore, when we train $\widehat{F},$ the residuals $\boldsymbol{\zeta}_{1:T}$ do not incur bias to $F,$ making it possible for the estimator $\widehat{F}$ to closely approximate $F.$ On the other hand, define the function $R$ as in \eqref{eq.def_R}, define $\boldsymbol{\gamma}_t = R(\boldsymbol{\zeta}_t),$ then the $i$-th element of $\boldsymbol{\gamma}_t$ is 
\begin{equation}
    \boldsymbol{\gamma}_{t,i} = \log\left(
    G_i^2(\boldsymbol{\zeta}_{t-1},\cdots,\boldsymbol{\zeta}_{t-s})
    \right) + \log\left(\boldsymbol{\eta}_{t,i}^2\right).
    \label{eq.autoregressive_gamma_i}
\end{equation}
Furthermore, by assuming that the functions $G_i^2(\cdot), i = 1,\cdots, d,$ depend on $\boldsymbol{\zeta}_{t-1},\cdots,\boldsymbol{\zeta}_{t-s}$ only through their element-wise squares, and notice that $\boldsymbol{\zeta}_{t,i}^2 = \exp\left(\boldsymbol{\gamma}_{t,i}\right),$ \eqref{eq.autoregressive_gamma_i} implies that 
\begin{equation}
\boldsymbol{\gamma}_t = A(\boldsymbol{\gamma}_{t-1},\cdots, \boldsymbol{\gamma}_{t-s}) + \boldsymbol{\iota}_t,
\label{eq.transform}
\end{equation}
where 
$A:\mathbf{R}^{d\times s}\to \mathbf{R}^d$ is a function such that $A_i(\boldsymbol{\gamma}_{t-1},\cdots, \boldsymbol{\gamma}_{t-s}) = \log\left(
    G_i^2(\boldsymbol{\zeta}_{t-1},\cdots,\boldsymbol{\zeta}_{t-s})
    \right) + \mathbf{E}\left[\log\left(\boldsymbol{\eta}_{t,i}^2\right)\right]$ and $\boldsymbol{\iota}_{t,i} =  \log\left(\boldsymbol{\eta}_{t,i}^2\right) - \mathbf{E}\left[\log\left(\boldsymbol{\eta}_{t,i}^2\right)\right].$
Therefore, the representation \eqref{eq.transform} allows the use of a mean-forecasting algorithm to learn $A,$ but inevitably incurs a constant bias term $\mathbf{E}\left[\log\left(\boldsymbol{\eta}_{t,i}^2\right)\right].$

\begin{figure}
    \centering
    \includegraphics[width=0.4\textwidth]{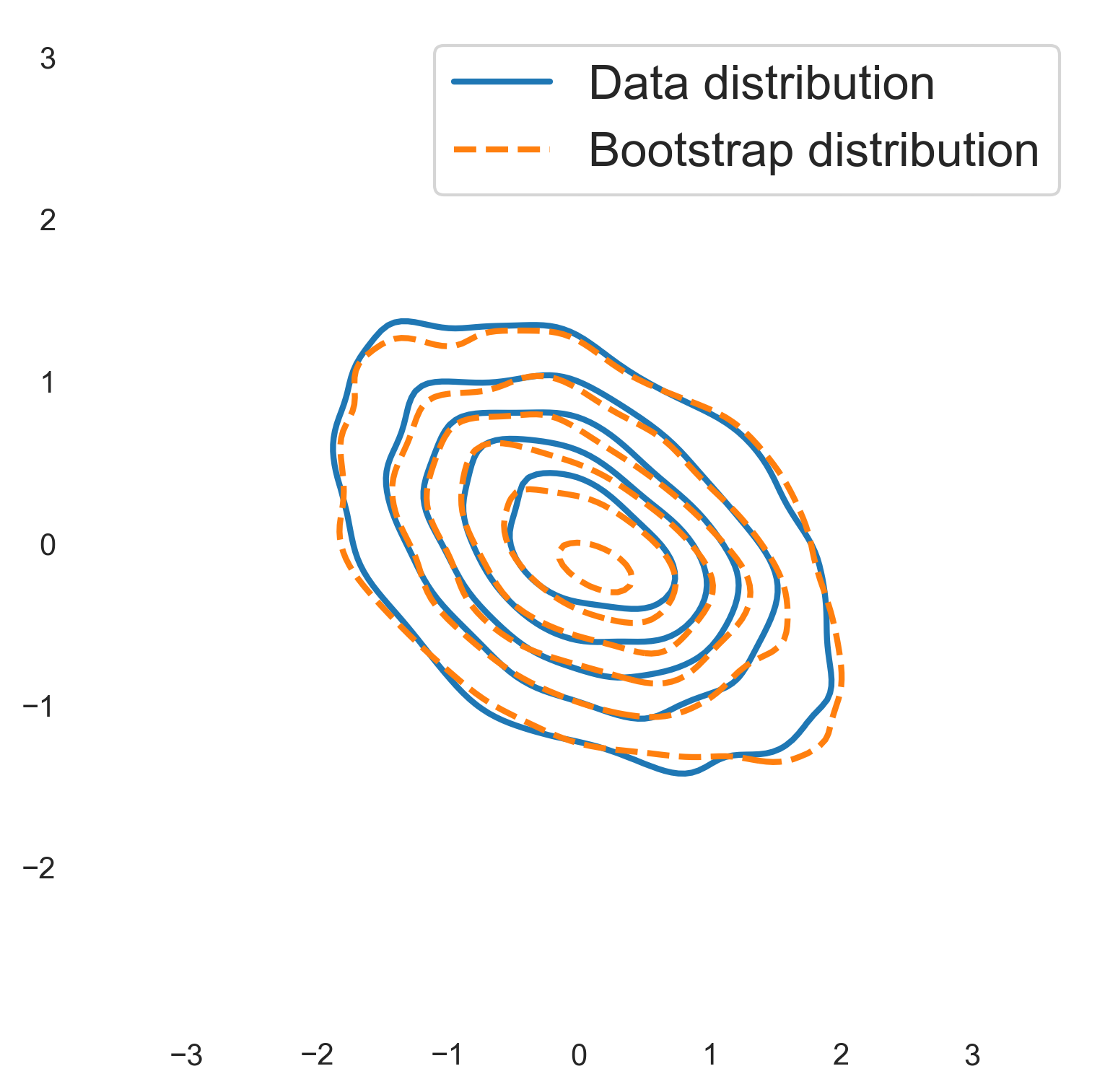}
    \caption{Kernel density estimates of the data (blue) and bootstrapped samples (orange dashed). The data are generated from a bivariate normal distribution; the bootstrap samples are obtained via Line 3 of Algorithm \ref{alg.pred}. The close agreement between the two curves indicates that the bootstrap samples effectively recover the underlying data distribution, even in the presence of element-wise dependence.}
\label{fig:cdf_edf}
\end{figure}

Fortunately, the constant bias does not affect the prediction as it is self-eliminated during \eqref{eq.normalize} of Algorithm \ref{alg.train}, which divides the fitted residuals $\widehat{\boldsymbol{\zeta}}_{(q+1):\ T}$ by $\widehat{G},$  and \eqref{eq.sample} of Algorithm \ref{alg.pred}, which multiplies the sampled $\boldsymbol{\eta}_{1:J}^*$ by $\widehat{G}.$

We emphasize that the assumption of $G^2_i$ depending on $\boldsymbol{\zeta}_{(t-s):\ (t-1)}$ through their element-wise squares is common in the literature. For example, the 
ARMA-GARCH models in \cite{Ling_McAleer_2003} leveraged this assumption. The advantage of this transformation is that, by replacing $\boldsymbol{\gamma}_t$ with $\widehat{\boldsymbol{\gamma}}_t = R(\widehat{\boldsymbol{\zeta}}_t),$ $\widehat{\boldsymbol{\gamma}}_t$ approximately follows an additive autoregressive process \eqref{eq.transform}, allowing the use of various conditional mean forecasting methods---such as those in \cite{lin2024cyclenet}---for estimating the function $A$ in \eqref{eq.transform}.

\subsection{Validity of the resampling procedure}

As illustrated in Section \ref{section.setting}, the validity of Algorithm \ref{alg.pred} 
comes from simulating the underlying data generating process of $\mathbf{x}_{1:(T+J)}$. Therefore,
if model \eqref{eq.heterogeneous_autoregressive} holds true and Algorithm \ref{alg.train} generates good estimators for $F$  and $G$ (up to a constant scale), the validity of  Algorithm \ref{alg.pred} is achieved provided that the empirical process of the vector  $\widehat{\boldsymbol{\eta}}_{(s+q+1):T}$---characterized by the probability measure defined by the joint CDF
in \eqref{eq.F_hat}---converges to the distributions of $\boldsymbol{\eta}_{1:\ (T+J)}.$  Theorem~\ref{theorem.consistent} provides a theoretical justification for this claim, with the proof postponed to Section~\ref{section.proof_theorem} in  Appendix.

\begin{theorem}
\label{theorem.consistent}
Suppose Assumptions 1 - 3 detailed in Section \ref{section.proof_theorem} in Appendix hold true. Then we have 
\begin{equation}
    \sup_{\boldsymbol{y}\in\mathbf{R}^d}\vert\widehat{P}(\boldsymbol{y}) - P(\boldsymbol{y})\vert\to_p 0,
    \label{eq.delta_F}
\end{equation}
where $\to_p$ denotes convergence in probability, $P(\cdot)$ denotes the CDF of $\boldsymbol{\eta}_{t},$ and the convergence is with respect to the sample size $T\to\infty$.
\end{theorem}

\begin{remark}
    By assuming $F$ and $G_i, i = 1,\cdots, d$ obey some parametric forms, Assumption 2 can actually be converted to the consistency of the estimated parameters, which is a common condition in the literature such as \cite{NIPS2016_1ce927f8, banerjee2026small}. While the assumption that $\boldsymbol{\eta}_t$ follows an identical distribution may seem restrictive for the validity of ProbRes, the ablation study in Table  \ref{Table.experiment_result_test_distribution_shift} in Appendix demonstrates that ProbRes exhibits only a modest performance decrease under moderate distributional shifts in future observations, highlighting its robustness in practical applications.
\end{remark}

Theorem \ref{theorem.consistent} guarantees that the distribution of the resampled normalized residuals $\boldsymbol{\eta}_{1:J}^*$ in Algorithm \ref{alg.pred} matches that of the true normalized residuals $\boldsymbol{\eta}_{(T+1):(T+J)}$. As a result, Algorithm \ref{alg.pred} effectively captures the distributional information of $\boldsymbol{\eta}_{(T+1):(T+J)}$.

\begin{table}
  \caption{Experimental results on synthetic data. Bold numbers indicate the lowest mean value, while underlined values denote the second-lowest mean among the four methods. Numbers in brackets represent the 95\% margins of error (e.g., the 95\% confidence interval for DLinear + quantile in terms of CRPS is $0.849 \pm 0.021$). All results are computed from five repetitions.}
  \label{Table.experiment_result_synthetic_data}
  \setlength{\tabcolsep}{3pt}
\scriptsize
  \centering
  \begin{tabular}{l | c c c |c c c}
    \toprule
     Dataset &\multicolumn{3}{c|}{GARCH(1, 1)} &\multicolumn{3}{c}{GARCH-in-mean(1,1)} \\
    \midrule
    Metrics & CRPS & MSIS & $\mathrm{ACE}_{90}$ & CRPS & MSIS & $\mathrm{ACE}_{90}$  \\
    \midrule
    Dlinear + quantile & 0.849(0.021) & 7.915(2.303)  & 0.039(0.004) & 0.801(0.009) & \underline{5.937(0.499)} & 0.036(0.013)  \\
     + CQR       & 0.895(0.023) &    \underline{7.728(1.766)} & \textbf{0.016(0.005)} & 0.818(0.002) & 6.511(0.654) & 0.043(0.023)\\
     + bootstrap & \underline{0.834(0.017)} & 8.478(3.776) & 0.034(0.022) &  \underline{0.798(0.007)} & 5.973(0.660) & \underline{0.026(0.023)}  \\
     + ProbRes &  \textbf{0.831(0.008)} & \textbf{7.389(2.405)} & \underline{0.019(0.016)} & \textbf{0.794(0.012)} & \textbf{5.502(0.579)} & \textbf{0.024(0.016)}  \\

    \midrule
    PatchTST + quantile & \underline{0.837(0.020)} & \underline{7.650(2.517)} & 0.031(0.020) & \underline{0.797(0.006)} & \underline{5.352(0.411)} & 0.012(0.008) \\
    + CQR        & 0.866(0.022) & 7.684(1.878) & \underline{0.018(0.007)} &0.807(0.009) &6.002(1.283) & \underline{0.010(0.006)}\\
     + bootstrap & 0.838(0.025) & 8.519(4.171) & 0.046(0.033) & 0.802(0.001) & 5.933(0.896) & 0.018(0.023) \\
     + ProbRes & \textbf{0.831(0.013)} & \textbf{7.413(2.479)} & \textbf{0.017(0.011)} & \textbf{0.795(0.008)} & \textbf{5.273(0.217)} & \textbf{0.009(0.008)} \\

    \midrule

    TimeMixer + quantile & 0.850(0.028) & \underline{8.245(4.107)} & 0.048(0.015) & \underline{0.797(0.008)} & \underline{5.742(0.438)} & 0.024(0.011)  \\
+CQR   &  0.907(0.029) & 8.849(2.246) & \underline{0.019(0.005)} &0.820(0.012) &6.248(1.218) &\textbf{0.006(0.003)}\\
     + bootstrap & \underline{0.834(0.022)} & 8.506(4.142) & 0.045(0.031) & 0.799(0.004) & 6.048(0.605) & 0.024(0.019)  \\

     + ProbRes & \textbf{0.827(0.012)} & \textbf{7.138(2.171)} & \textbf{0.018(0.014)} & \textbf{0.791(0.008)} & \textbf{5.668(0.612)} & \underline{0.018(0.014)} \\
    \bottomrule
  \end{tabular}
  
\end{table}

\section{Numerical experiments}

This section empirically supports our claim that explicitly modeling volatility improves distributional and interval forecasting performance under heteroskedasticity, as demonstrated through both synthetic and real-life data experiments. Due to space limitations, detailed experimental settings---including dataset descriptions, hyperparameter configurations, and evaluation metrics---as well as additional experimental results, are postponed to Section \ref{section.additional_experimental} in Appendix.


\textbf{Metrics.} The evaluation metrics are CRPS, MSIS, and $\mathrm{ACE}_{90}$; detailed definitions are provided in Section~\ref{section.metrics} of the online supplement. These metrics respectively assess the alignment of predictive distributions with future observations, the accuracy of prediction intervals, and coverage calibration. A well-performing probabilistic forecasting method attains low values across all three metrics.

\textbf{Backbone architectures and baselines.} The backbone architectures include DLinear \cite{Zeng_Chen_Zhang_Xu_2023}, PatchTST \cite{nie2023a}, and TimeMixer \cite{wang2023timemixer}. These models span distinct architectural paradigms: DLinear employs separate linear components to model trend and seasonal patterns in time series; PatchTST is based on a Transformer architecture; and TimeMixer adopts an MLP-based design. Evaluating across these backbones ensures that the observed improvements of our method are not tied to specific architectures. Notably, all three methods are originally designed for point forecasting and cannot directly produce probabilistic forecasts without additional wrappers.  

\begin{table*}[t]
  \caption{Experimental results on real-life data. Meaning of notations of this table coincides with Table \ref{Table.experiment_result_synthetic_data}. Tables \ref{Table.experiment_result_real_life_data_appendix_1} and \ref{Table.experiment_result_real_life_data_appendix_2} in Appendix  contain  95\% margins of error. }
  \label{Table.experiment_result_real_life_data_main}
  \setlength{\tabcolsep}{3pt}
  \small
  \centering
  \begin{tabularx}{\textwidth}{l | X X X | X X X | X X X}        
    \toprule
     Dataset &\multicolumn{3}{c|}{Exchange} &\multicolumn{3}{c|}{S\&P 500 Industrial} &\multicolumn{3}{c}{Electricity}\\
    \midrule
    Metrics & CRPS & MSIS & $\mathrm{ACE}_{90}$ & CRPS & MSIS & $\mathrm{ACE}_{90}$  & CRPS & MSIS & $\mathrm{ACE}_{90}$\\
    \midrule
    Dlinear + quantile  & \underline{0.015} & 38.02 & \underline{0.053} & \underline{0.817} & 5.494 & 0.048  & 0.065 & 8.380 & \underline{0.065}\\
    + bootstrap  & 0.017 & 40.71 & 0.235 & 0.826 & 5.605 & \textbf{0.012}  & \underline{0.057} & \underline{7.986} & 0.096\\
    + CQR & \underline{0.015} & \underline{30.50} & 0.059 & 0.833 & \underline{5.333} & 0.031 &0.166 &17.18 &0.087\\
    + ProbRes  & \textbf{0.010} & \textbf{20.83} & \textbf{0.048} & \textbf{0.816} & \textbf{5.222} & \underline{0.030}  & \textbf{0.054} & \textbf{7.602} & \textbf{0.055}\\
    \midrule
    PatchTST + quantile & \underline{0.013} & \underline{38.65} & 0.104 & \underline{0.813} & 5.217 & 0.044  & 0.072 & 6.758 & \underline{0.044} \\
     + bootstrap        & 0.026 & 101.5 & 0.365 & 0.821 & 5.537 & \textbf{0.013}   &   \underline{0.064}  & \underline{6.599} & 0.081\\
     + CQR        & 0.016 & 42.03 & \underline{0.060} & 0.821 & \underline{5.190} & 0.032  & 0.106  & 10.48 & 0.086\\
     + ProbRes          & \textbf{0.012} & \textbf{22.73} & \textbf{0.035} & \textbf{0.811} & \textbf{5.170} & \underline{0.024} &  \textbf{0.063} & \textbf{5.825} & \textbf{0.041}\\
     \midrule
      TimeMixer + quantile & 0.029 & 77.00 & 0.056 & \textbf{0.803} & \textbf{4.966} & 0.031   & \underline{0.254} & 14.47 & 0.064\\
      + bootstrap   &  0.016 & \underline{63.22} & \underline{0.026} & 0.822 & 5.653 & \underline{0.018}   & 0.271 & \textbf{9.999} & 0.095\\
      + CQR   & 0.029 & 67.61 & 0.070 & 0.813 & \underline{5.012} & 0.043   & 0.258 & 12.67 & \underline{0.030}\\
      + ProbRes     & \textbf{0.013} & \textbf{32.64} & \textbf{0.015} & \textbf{0.803} & 5.073 & \textbf{0.012}   & \textbf{0.235} & \underline{10.67} & \textbf{0.007}\\
    \bottomrule
  \end{tabularx}
\end{table*}

The baselines include quantile regressions, i.e., training with the pinball loss \cite{10.3150/10-BEJ267}; the autoregressive bootstrap \cite{PAN20161, https://doi.org/10.1111/jtsa.12739}, and the conformalized quantile regression (CQR)  \cite{romano2019conformalized}. These baselines are also compatible with a wide range of model architectures. However, as noted in \cite{FAN2019261}, they do not explicitly model conditional heteroskedasticity, making them natural counterparts to our method.

\subsection{Synthetic data experiment}

The synthetic data are generated from GARCH(1,1) and GARCH-in-mean(1,1) processes with heavy-tailed  innovations (t-distribution with 5 degrees of freedom), which respectively reflect  volatility clustering and volatility-mean interaction effects in financial time series \cite{tsay2005analysis}, along with strong non-Gaussianity. Both datasets exhibit strong heteroskedasticity. A detailed description of the data-generating processes is provided in Section \ref{section.data_set} of the Appendix.

 Experimental results in Table \ref{Table.experiment_result_synthetic_data} show that ProbRes consistently achieves strong performance in both forecasting future distributions and producing accurate, well-calibrated prediction intervals across most datasets and backbone models. These gains suggest that its improvements are architecture-agnostic and stem from better modeling of conditional heteroskedasticity. In contrast, baselines that do not explicitly model conditional heteroskedasticity either produce reasonably accurate predictive distributions but excessively wide prediction intervals, or fail to adequately capture the future distribution.

\subsection{Real-data experiments}

\textbf{Univariate probabilistic forecasting: }  Table~\ref{Table.experiment_result_real_life_data_main}, together with Tables~\ref{Table.experiment_result_real_life_data_appendix_1} and~\ref{Table.experiment_result_real_life_data_appendix_2} in the Appendix, reports probabilistic forecasting results on three real-world datasets. Compared with the baselines, ProbRes achieves the best performance on most evaluation metrics for both distributional accuracy and prediction interval quality, across all datasets and model architectures. This indicates that the improvements of ProbRes are architecture-agnostic, highlighting the benefit of explicitly modeling conditional heteroskedasticity in probabilistic forecasting.

In addition, the validity of ProbRes does not rely on specific parametric assumptions about the distribution of normalized residuals, which further contributes to its robustness. As illustrated in Figure~\ref{figure.histgram}, normalized residuals in real-world datasets rarely follow standard parametric families such as Gaussian or Student-t distributions, and may exhibit multimodality or heavy tails. Such distributional properties can harm the performance of baseline methods \cite{JMLR:v26:24-0589}.
\begin{figure}[t]
\centering
\begin{subfigure}{0.45\textwidth}
    \centering
    \includegraphics[width=\linewidth]{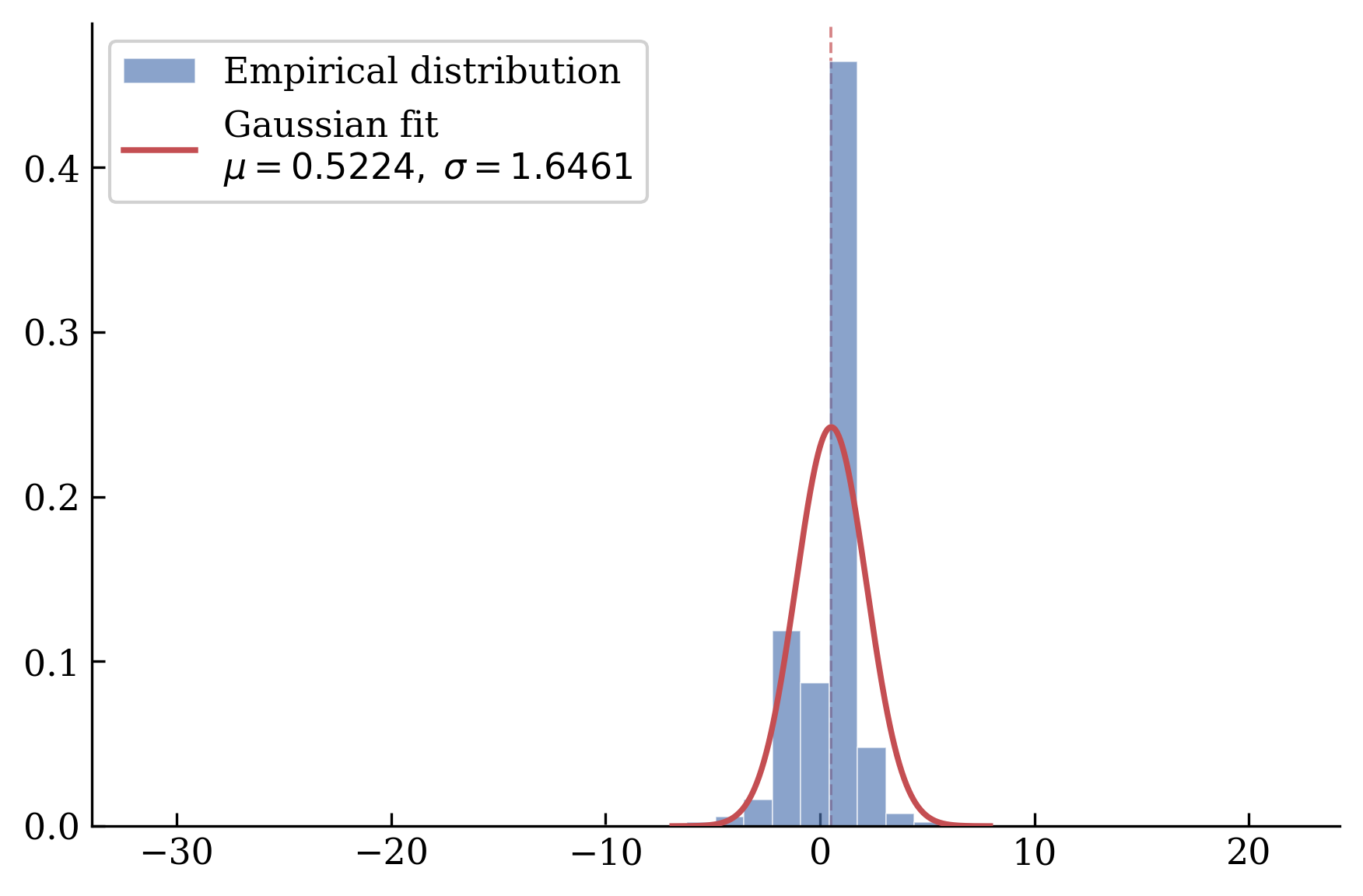}
    \caption{Exchange}
\end{subfigure}
\hfill
\begin{subfigure}{0.45\textwidth}
    \centering
    \includegraphics[width=\linewidth]{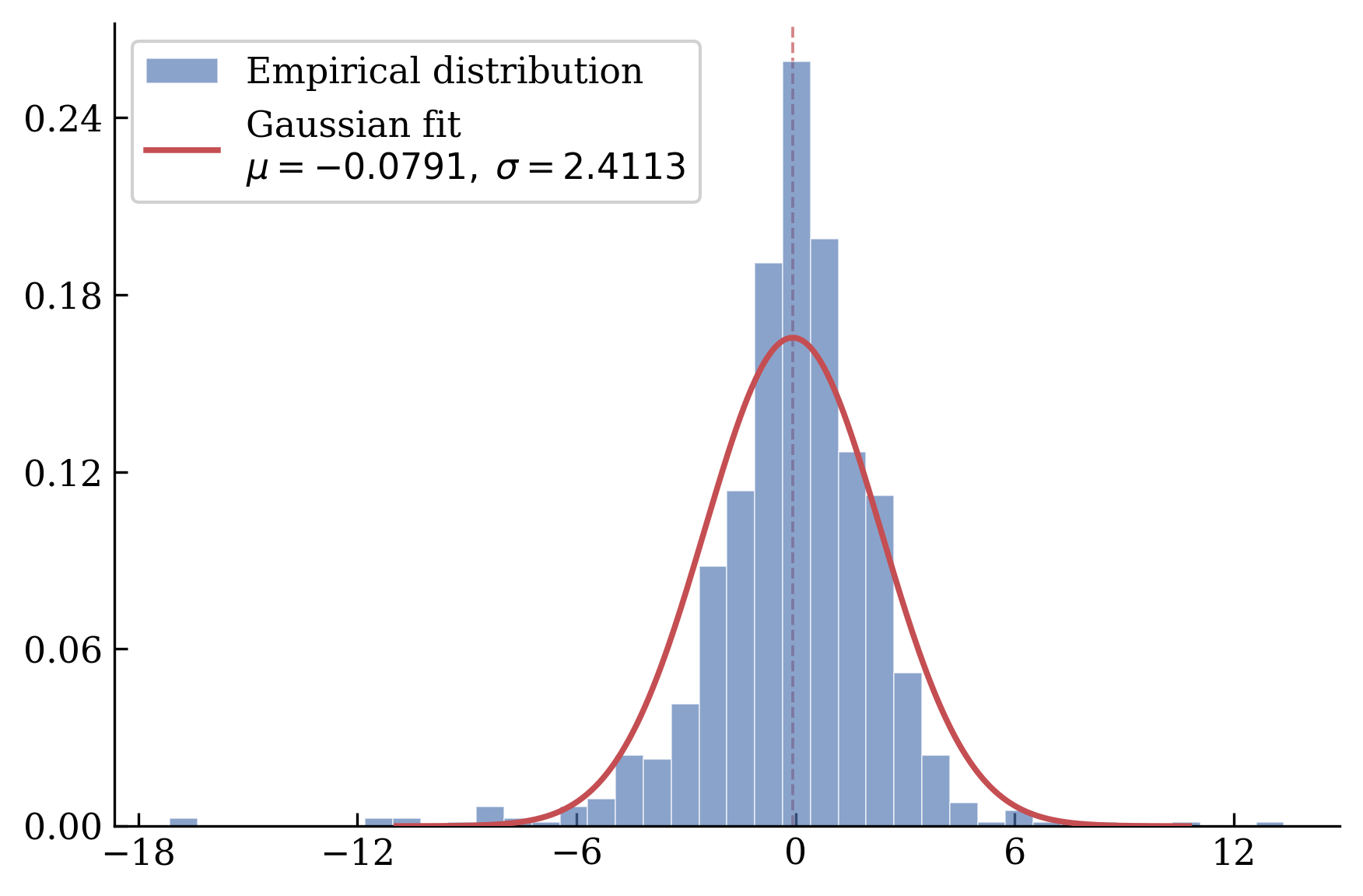}
    \caption{S\&P 500  Industrial}
\end{subfigure}
\caption{Histograms of the normalized fitted residuals $\widehat{\boldsymbol{\eta}}_t$ for the \textit{Exchange} and \textit{S\&P 500 Industrial} datasets, together with the corresponding fitted Gaussian density curves. The histograms indicate that the distribution of $\widehat{\boldsymbol{\eta}}_t$ either exhibits strong concentration around zero (Exchange) or displays heavy tails (S\&P 500 Industrial). These deviations from normality stress the necessity of accounting for residual distributional characteristics in probabilistic forecasting.}
\label{figure.histgram}
\end{figure}

\begin{wraptable}{r}{0.55\textwidth}
  \caption{Numerical experiment results on multivariate time series datasets. The interpretation of the values and the use of boldface are the same as in Table \ref{Table.experiment_result_synthetic_data}. The full table with 95\% margins of error is provided in Table \ref{Table.experiment_result_multivariate_appendix} in Appendix.}
  \vspace{-5pt}
  \setlength{\tabcolsep}{2pt}
  \scriptsize
  \label{Table.experiment_result_multivariate}
  \centering
  \begin{tabular}{l|c c| c c | c c}
    \toprule
    \multicolumn{1}{l|}{Dataset} & \multicolumn{2}{c}{ETTh1} & \multicolumn{2}{c}{ETTh2}  &  \multicolumn{2}{c}{Electricity}\\
    \midrule
    Metrics     & CRPS$_{sum}$  & ES   &CRPS$_{sum}$    & ES &CRPS$_{sum}$  & ES \\
    \midrule
   VEC-LSTM & 0.184 & 3.873 & 0.095 & 6.423 & 0.441  & 48684\\
    +ProbRes   & \textbf{0.182}  & \textbf{3.503} & \textbf{0.087} & \textbf{6.067} & \textbf{0.301} &  \textbf{41398}   \\
    \midrule
    TMDM  & 0.456  & 13.344 & 0.092 & \textbf{6.933} & 0.655
      & 87761
    \\
    +ProbRes & \textbf{0.397}  & \textbf{11.341} & \textbf{0.092}  & 7.326 & \textbf{0.292}  & \textbf{37322}\\
    \bottomrule
  \end{tabular}

  \textbf{Remark:}   CRPS$_{sum}$ here denotes the sum of element-wise CRPS values, a commonly used metric for evaluating multivariate probabilistic forecasts \cite{NEURIPS2019_0b105cf1}.
\vspace{-5pt}
\end{wraptable}

\textbf{Multivariate probabilistic forecasting: } We conduct multivariate probabilistic forecasting experiments to evaluate ProbRes's ability to generate accurate probabilistic forecasts in the presence of element-wise dependence. The experiment uses three  real-world datasets: \textit{ETTh1, ETTh2, Electricity} (see Section \ref{section.data_set} in Appendix for dataset details). In addition to the sum of element-wise CRPS values \cite{NEURIPS2019_0b105cf1}, we also report the energy score (ES) \cite{chung2024samplingbased} (see Section \ref{section.metrics} in the Appendix) to evaluate multivariate probabilistic forecasts by capturing spatial dependence.

The results are demonstrated in Table \ref{Table.experiment_result_multivariate}, using baselines \textit{VEC-LSTM} \cite{NEURIPS2019_0b105cf1} and \textit{TMDM} \cite{li2024transformermodulated}. Since the baselines are probabilistic forecasting methods, we apply ProbRes only to their mean forecasts to ensure a fair comparison. ProbRes achieves improvements across all metrics for VEC-LSTM and on most metrics for TMDM (e.g., the CRPS of TMDM decreases from 0.655 to 0.292 after incorporating ProbRes on the \textit{Electricity} dataset). The result highlights ProbRes’s ability to capture spatial dependence in multivariate time series beyond modeling conditional heteroskedasticity.

\section{Discussion}
\label{section.discussion}

This paper introduces ProbRes, a post-hoc probabilistic calibration method that jointly models the conditional mean, volatility, and the distribution of fitted residuals, and leverages these components for forecasting via resampling. By explicitly capturing volatility dynamics and employing resampling, ProbRes is suitable for handling heteroskedastic data with complex residual distributions. We further provide theoretical guarantees supporting the validity of ProbRes.

In addition, ProbRes is flexible with respect to the choice of backbone models. It can be applied to existing point forecasting models to produce probabilistic forecasts, without requiring modifications to model architectures or loss functions.

\textbf{Limitations and Future Work.} One main limitation of our work lies in  the computational complexity of the algorithm: 
Suppose the rounds  of bootstrap repetition is $B$ and the prediction length is $J,$ then the computational complexity in forecasting is of order $BJ.$  Concerning this, one potential future direction of this work involves leveraging advanced subsampling techniques, like those in \cite{10.1093/biomet/asae039}, to decrease computational complexity.

\appendix

\section{Proof of Theorem \ref{theorem.consistent}}
\label{section.proof_theorem}
To validate Theorem \ref{theorem.consistent}, we propose the following technical assumptions.

\textbf{Assumptions: }

1. $\boldsymbol{\eta}_{t}, t = 1,2,\cdots,$ are independent and identically distributed with continuous cumulative distribution function $P(\cdot):\mathbf{R}^d\to \mathbf{R}$. Suppose $\mathbf{E}\left[\boldsymbol{\eta}_{1}\right] = 0$ and $\mathrm{Var}(\boldsymbol{\eta}_{1,i})\leq C$ for a constant $C$ and any $i = 1,\cdots, d$.

2. For a vector $\mathbf{x}\in\mathbf{R}^d,$ define  $|| \mathbf{x} || $ as its Euclidean norm. We suppose  the conditional mean and volatility function estimator satisfy 
\begin{align*}
    \sup_{\mathbf{Y}\in\mathbf{R}^{d\times q}}||\widehat{F}(\mathbf{Y}) - F(\mathbf{Y})||\to_p 0\quad
    \text{and}\quad\sup_{\mathbf{Y}\in\mathbf{R}^{d\times s}}\vert \widehat{G}_i(\mathbf{Y}) - G_i(\mathbf{Y})\vert\to_p 0,
\end{align*}
where $i = 1,2,\cdots,d,$ and $\to_p$ denotes convergence in probability.  

3. Suppose $G_i(\cdot)$ is continuous differentiable with bounded gradient, i.e.,
$$
\sup_{\mathbf{Y}\in\mathbf{R}^{d\times s}} || \nabla_\mathbf{Y} G_i(\mathbf{Y})|| < \infty
$$
for $i = 1,\cdots, d.$ Furthermore,  suppose there exists a constant $c > 0$ such that 
$$
\inf_{\mathbf{Y}\in\mathbf{R}^{d\times s}} \vert\displaystyle G_i(\mathbf{Y})\vert > c
$$
for $i = 1,\cdots, d.$

With those assumptions, we demonstrate that Theorem \ref{theorem.consistent} holds true.

\begin{proof}[Proof of Theorem \ref{theorem.consistent}]
    For any vector $\mathbf{y} = (\mathbf{y}_1,\cdots,\mathbf{y}_d)^\top\in\mathbf{R}^d,$ define 
    $$\widetilde{P}(\boldsymbol{y}) = \frac{1}{T - q - s}\sum_{t = s + q + 1}^T\mathbf{1}_{\boldsymbol{\eta}_{t}\leq \boldsymbol{y}}.
    $$
    From Glivenko-Cantelli Theorem, like Theorem 4 of \cite{Sharipov2011}, we have 
    \begin{align*}
        \sup_{\boldsymbol{y}\in\mathbf{R}^d}\vert \widetilde{P}(\boldsymbol{y}) - P(\boldsymbol{y})\vert\to_p 0.
    \end{align*}
    On the other hand, define the functions 
    $$
    g_0(u) = (1 - \min(1, \max(u,0))^4)^4\quad\text{and }\quad g_{\psi,t}(x) = g_0(\psi(x - t)),
    $$
    as demonstrated in \cite{10.1093/biomet/asz020}, which satisfy the following property: $g_0(\cdot)$ is third-order continuous differentiable, $g_0(u) = 1$ if $u\leq 0$, $g_0(u) = 0$ if $u\geq 1$, and  
    \begin{align*}
        g_* = \sup_{u\in\mathbf{R}}\{
        \vert
        g_0^\prime(u)
        \vert + \vert
        g_0^{\prime\prime}(u)
        \vert + \vert
        g_0^{\prime\prime\prime}(u)
        \vert
        \} < \infty,\ \mathbf{1}_{x\leq t}\leq g_{\psi,t}(x)\leq \mathbf{1}_{x\leq t + \psi^{-1}},
        \sup_{x,t\in\mathbf{R}}\vert g_{\psi,t}^\prime (x)\vert\leq g_*\psi.
    \end{align*}
    Define 
    \begin{align*}
        \boldsymbol{\Delta}_t & = \widehat{\boldsymbol{\eta}}_t - \boldsymbol{\eta}_t\\
        &= \widehat{G}^{-1}\left(\widehat{\boldsymbol{\zeta}}_{t-s},\cdots \widehat{\boldsymbol{\zeta}}_{t-1}\right)\left(F(\mathbf{x}_{t-q},\cdots,\mathbf{x}_{t-1}) - \widehat{F}(\mathbf{x}_{t-q},\cdots,\mathbf{x}_{t-1})\right)\\
        & + \widehat{G}^{-1}\left(\widehat{\boldsymbol{\zeta}}_{t-s},\cdots \widehat{\boldsymbol{\zeta}}_{t-1}\right)\left(
        G\left(\boldsymbol{\zeta}_{t-s},\cdots \boldsymbol{\zeta}_{t-1}\right) - \widehat{G}\left(\widehat{\boldsymbol{\zeta}}_{t-s},\cdots \widehat{\boldsymbol{\zeta}}_{t-1}\right)
        \right)\boldsymbol{\eta}_t.
    \end{align*}
    From the definition \eqref{eq.F_hat}
    \begin{align*}
        \widehat{P}(\boldsymbol{y}) = \frac{1}{T - q - s}\sum_{t = s + q + 1}^T \mathbf{1}_{\boldsymbol{\eta}_t +\boldsymbol{\Delta}_t\leq \boldsymbol{y} }\leq \frac{1}{T - q - s}\sum_{t = s + q + 1}^T\prod_{i = 1}^d 
        g_{\psi,\boldsymbol{y}_i}(\boldsymbol{\eta}_{t,i} +\boldsymbol{\Delta}_{t,i}).
    \end{align*}
    From Taylor expansion,
    \begin{align*}
        &\vert \prod_{i = 1}^d 
        g_{\psi,\boldsymbol{y}_i}(\boldsymbol{\eta}_{t,i} +\boldsymbol{\Delta}_{t,i})
        - \prod_{i = 1}^d 
        g_{\psi,\boldsymbol{y}_i}(\boldsymbol{\eta}_{t,i})
        \vert\\
        &\leq \sum_{i = 1}^d (\prod_{j = 1}^{i - 1}g_{\psi,\boldsymbol{y}_i}(\boldsymbol{\eta}_{t,i} +\boldsymbol{\Delta}_{t,i})(g_{\psi,\boldsymbol{y}_i}(\boldsymbol{\eta}_{t,i} +\boldsymbol{\Delta}_{t,i}) - g_{\psi,\boldsymbol{y}_i}(\boldsymbol{\eta}_{t,i}))\prod_{j = i + 1}^dg_{\psi,\boldsymbol{y}_i}(\boldsymbol{\eta}_{t,i})\\
        &\leq \sum_{i = 1}^d\vert g_{\psi,\boldsymbol{y}_i}(\boldsymbol{\eta}_{t,i} +\boldsymbol{\Delta}_{t,i}) - g_{\psi,\boldsymbol{y}_i}(\boldsymbol{\eta}_{t,i})\vert\leq g_*\psi\sum_{i = 1}^d\vert \boldsymbol{\Delta}_{t,i}\vert\leq g_*\psi\sqrt{d}
        ||
        \boldsymbol{\Delta}_{t}
        ||.
    \end{align*}
    Therefore, 
    \begin{align*}
        &\frac{1}{T - q - s}\sum_{t = s + q + 1}^T\prod_{i = 1}^d 
        g_{\psi,\boldsymbol{y}_i}(\boldsymbol{\eta}_{t,i} +\boldsymbol{\Delta}_{t,i})\\
        &\leq \frac{1}{T - q - s}\sum_{t = s + q + 1}^T\prod_{i = 1}^d 
        g_{\psi,\boldsymbol{y}_i}(\boldsymbol{\eta}_{t,i}) 
        + \frac{g_*\psi\sqrt{d}}{T - q - s}\sum_{t = s + q + 1}^T||
        \boldsymbol{\Delta}_{t}
        ||\\
        &\leq \frac{1}{T - q -  s}\sum_{t = s + q + 1}^T \mathbf{1}_{\boldsymbol{\eta}_t\leq \boldsymbol{y} + \psi^{-1}}
        + \frac{g_*\psi\sqrt{d}}{T - q - s}\sum_{t = s + q + 1}^T ||\boldsymbol{\Delta}_t||\\
        &= \widetilde{P}(\boldsymbol{y} + \psi^{-1}\mathbf{h}) +  \frac{g_*\psi\sqrt{d}}{T - q - s}\sum_{t = s+ q + 1}^T ||
        \boldsymbol{\Delta}_{t}
        ||,
    \end{align*}
    where $\mathbf{h} = (1,1,\cdots, 1)^\top$. Similarly, 
    \begin{align*}
        \widehat{P}(\boldsymbol{y}) &\geq \frac{1}{T - q -  s}\sum_{t = s + q + 1}^T\prod_{i = 1}^d 
        g_{\psi,\boldsymbol{y}_i - \psi^{-1}}(\boldsymbol{\eta}_{t,i} +\boldsymbol{\Delta}_{t,i})\\
        &\geq \frac{1}{T - q - s}\sum_{t = s + q + 1}^T\prod_{i = 1}^d 
        g_{\psi,\boldsymbol{y}_i - \psi^{-1}}(\boldsymbol{\eta}_{t,i}) 
        - \frac{g_*\psi\sqrt{d}}{T - q - s}\sum_{t = s + q + 1}^T||\boldsymbol{\Delta}_t||\\
        &\geq \widetilde{P}(\boldsymbol{y} - \psi^{-1}\mathbf{h}) -  \frac{g_*\psi\sqrt{d}}{T - q - s}\sum_{t = s + q + 1}^T ||\boldsymbol{\Delta}_t||.
    \end{align*}
With probability tending to $1$, 
    \begin{align*}
\inf_{\mathbf{Y}\in\mathbf{R}^{d\times s}}\widehat{G}_i(\mathbf{Y})
\geq \inf_{\mathbf{Y}\in\mathbf{R}^{d\times s}}G_i(\mathbf{Y}) - \sup_{\mathbf{Y}\in\mathbf{R}^{d\times s}}\vert \widehat{G}_i(\mathbf{Y}) - G_i(\mathbf{Y})\vert > c/2. 
    \end{align*}
If that happens for $i = 1,\cdots, d,$ we have  
\begin{equation}
\begin{aligned}
&||\widehat{G}^{-1}\left(\widehat{\boldsymbol{\zeta}}_{t-s},\cdots \widehat{\boldsymbol{\zeta}}_{t-1}\right)\left(F(\mathbf{x}_{t-q},\cdots,\mathbf{x}_{t-1}) - \widehat{F}(\mathbf{x}_{t-q},\cdots,\mathbf{x}_{t-1})\right)||\\
&\leq \frac{2}{c}\sup_{\mathbf{Y}\in\mathbf{R}^{d\times q}}||F(\mathbf{Y}) - \widehat{F}(\mathbf{Y})||\to_p 0.
\end{aligned}
\label{eq.f_difference}
\end{equation}
On the other hand, for any $i = 1,\cdots,d,$ the $i$th element of 

\noindent $\widehat{G}^{-1}\left(\widehat{\boldsymbol{\zeta}}_{t-s},\cdots \widehat{\boldsymbol{\zeta}}_{t-1}\right)\left(
        G\left(\boldsymbol{\zeta}_{t-s},\cdots \boldsymbol{\zeta}_{t-1}\right) - \widehat{G}\left(\widehat{\boldsymbol{\zeta}}_{t-s},\cdots \widehat{\boldsymbol{\zeta}}_{t-1}\right)
        \right)\boldsymbol{\eta}_t$ is 
\begin{align*}
    \frac{G_i\left(\boldsymbol{\zeta}_{t-s},\cdots \boldsymbol{\zeta}_{t-1}\right) - \widehat{G}_i\left(\widehat{\boldsymbol{\zeta}}_{t-s},\cdots \widehat{\boldsymbol{\zeta}}_{t-1}\right)}{\widehat{G}_i\left(\widehat{\boldsymbol{\zeta}}_{t-s},\cdots \widehat{\boldsymbol{\zeta}}_{t-1}\right)}\boldsymbol{\eta}_{t,i}.
\end{align*}
and
\begin{align*}
    &\vert
    \frac{G_i\left(\boldsymbol{\zeta}_{t-s},\cdots \boldsymbol{\zeta}_{t-1}\right) - \widehat{G}_i\left(\widehat{\boldsymbol{\zeta}}_{t-s},\cdots \widehat{\boldsymbol{\zeta}}_{t-1}\right)}{\widehat{G}_i\left(\widehat{\boldsymbol{\zeta}}_{t-s},\cdots \widehat{\boldsymbol{\zeta}}_{t-1}\right)}\boldsymbol{\eta}_{t,i}
    \vert\\
    &\leq \frac{2\vert \boldsymbol{\eta}_{t,i}\vert}{c}\left(\vert 
    G_i\left(\boldsymbol{\zeta}_{t-s},\cdots \boldsymbol{\zeta}_{t-1}\right) - G_i\left(\widehat{\boldsymbol{\zeta}}_{t-s},\cdots \widehat{\boldsymbol{\zeta}}_{t-1}\right)
    \vert\right.\\ 
    &\left. + \vert G_i\left(\widehat{\boldsymbol{\zeta}}_{t-s},\cdots \widehat{\boldsymbol{\zeta}}_{t-1}\right) - \widehat{G}_i\left(\widehat{\boldsymbol{\zeta}}_{t-s},\cdots \widehat{\boldsymbol{\zeta}}_{t-1}\right)\vert\right)
\end{align*}
From Assumption 2,
\begin{equation}
    \vert G_i\left(\widehat{\boldsymbol{\zeta}}_{t-s},\cdots \widehat{\boldsymbol{\zeta}}_{t-1}\right) - \widehat{G}_i\left(\widehat{\boldsymbol{\zeta}}_{t-s},\cdots \widehat{\boldsymbol{\zeta}}_{t-1}\right)\vert
    \leq \sup_{\mathbf{Y}\in\mathbf{R}^{d\times s}}\vert G_i\left(\mathbf{Y}\right) - \widehat{G}_i\left(\mathbf{Y}\right)\vert\to_p 0.
\label{eq.sigma_delta}
\end{equation}
On the other hand, for any $t = q + 1,\cdots, T,$
\begin{align*}
    ||
    \widehat{\boldsymbol{\zeta}}_{t} - \boldsymbol{\zeta}_{t}
    || &= || F(\mathbf{x}_{t-q},\cdots,\mathbf{x}_{t-1}) - \widehat{F}(\mathbf{x}_{t-q},\cdots, \mathbf{x}_{t-1})||\\
       &\leq \sup_{\mathbf{Y}\in\mathbf{R}^{d\times q}} || F(\mathbf{Y}) - \widehat{F}(\mathbf{Y})||\to_p 0.
\end{align*}
Define the matrix
\begin{align*}
    \boldsymbol{\Gamma} = \left[
    \begin{matrix}
        \widehat{\boldsymbol{\zeta}}_{t-s} - \boldsymbol{\zeta}_{t-s} & \cdots & \widehat{\boldsymbol{\zeta}}_{t-1} - \boldsymbol{\zeta}_{t-1}
    \end{matrix}
    \right],
\end{align*}
from Taylor's expansion,
\begin{equation}
\begin{aligned}
    \vert 
    G_i\left(\boldsymbol{\zeta}_{t-s},\cdots \boldsymbol{\zeta}_{t-1}\right) - G_i\left(\widehat{\boldsymbol{\zeta}}_{t-s},\cdots \widehat{\boldsymbol{\zeta}}_{t-1}\right)
    \vert &= 
    \vert
    \sum_{i = 1}^d\sum_{j = 1}^s(\nabla_\mathbf{Z} G_i(\mathbf{Z}))_{ij}\boldsymbol{\Gamma}_{ij}
    \vert\\
    & \leq \sum_{i = 1}^d\sum_{j = 1}^s \vert \nabla_\mathbf{Z} G_i(\mathbf{Z}))_{ij}\vert \vert \boldsymbol{\Gamma}_{ij}\vert\\
    & \leq Cds \sup_{\mathbf{Y}\in\mathbf{R}^{d\times q}} || F(\mathbf{Y}) - \widehat{F}(\mathbf{Y})||,
\end{aligned}
\label{eq.delta_zeta}
\end{equation}
where $\mathbf{Z}\in\mathbf{R}^{d\times s}$ is a random matrix. From eq.\eqref{eq.f_difference},  eq.\eqref{eq.sigma_delta} and eq.\eqref{eq.delta_zeta}, with probability tending to 1
\begin{align*}
    ||\boldsymbol{\Delta}_t|| &\leq  \frac{2}{c}\sup_{\mathbf{Y}\in\mathbf{R}^{d\times q}}||F(\mathbf{Y}) - \widehat{F}(\mathbf{Y})|| + \sqrt{\sum_{i = 1}^d \left(\frac{G_i\left(\boldsymbol{\zeta}_{t-s},\cdots \boldsymbol{\zeta}_{t-1}\right) - \widehat{G}_i\left(\widehat{\boldsymbol{\zeta}}_{t-s},\cdots \widehat{\boldsymbol{\zeta}}_{t-1}\right)}{\widehat{G}_i\left(\widehat{\boldsymbol{\zeta}}_{t-s},\cdots \widehat{\boldsymbol{\zeta}}_{t-1}\right)}\boldsymbol{\eta}_{t,i}\right)^2}\\
    &\leq \frac{2}{c}\sup_{\mathbf{Y}\in\mathbf{R}^{d\times q}}||F(\mathbf{Y}) - \widehat{F}(\mathbf{Y})|| + \frac{2\sqrt{d}}{c}\max_{i = 1,\cdots,d}\vert \boldsymbol{\eta}_{t,i}\vert\times\vert G_i\left(\boldsymbol{\zeta}_{t-s},\cdots \boldsymbol{\zeta}_{t-1}\right) - \widehat{G}_i\left(\widehat{\boldsymbol{\zeta}}_{t-s},\cdots \widehat{\boldsymbol{\zeta}}_{t-1}\right)\vert\\
    &\leq \frac{2}{c}\sup_{\mathbf{Y}\in\mathbf{R}^{d\times q}}||F(\mathbf{Y}) - \widehat{F}(\mathbf{Y})|| + \frac{2\sqrt{d}}{c}\left(\sum_{i = 1}^d\vert \boldsymbol{\eta}_{t,i}\vert\right)\left(\sup_{\mathbf{Y}\in\mathbf{R}^{d\times s}}\vert G_i\left(\mathbf{Y}\right) - \widehat{G}_i\left(\mathbf{Y}\right)\vert\right)\\
    &+ \frac{2\sqrt{d}}{c}\left(\sum_{i = 1}^d\vert \boldsymbol{\eta}_{t,i}\vert\right)\left(Cds \sup_{\mathbf{Y}\in\mathbf{R}^{d\times q}} || F(\mathbf{Y}) - \widehat{F}(\mathbf{Y})||\right).
\end{align*}
Since 
\begin{align*}
    \frac{\psi\sqrt{d}}{T - q - s}\sum_{t = s + q + 1}^T ||\boldsymbol{\Delta}_t|| & \leq \frac{2\psi\sqrt{d}}{c}\sup_{\mathbf{Y}\in\mathbf{R}^{d\times q}}||F(\mathbf{Y}) - \widehat{F}(\mathbf{Y})||\\
    & + \frac{2\psi d}{c(T - q - s)}\sum_{i = 1}^d \sup_{\mathbf{Y}\in\mathbf{R}^{d\times s}}\vert G_i\left(\mathbf{Y}\right) - \widehat{G}_i\left(\mathbf{Y}\right)\vert\sum_{t = s + q + 1}^T\vert\boldsymbol{\eta}_{t,i}\vert\\
    & + \frac{2C\psi d^2s}{c(T-q-s)}\sup_{\mathbf{Y}\in\mathbf{R}^{d\times q}} || F(\mathbf{Y}) - \widehat{F}(\mathbf{Y})||\sum_{i = 1}^d\sum_{t = s + q + 1}^T\vert\boldsymbol{\eta}_{t,i}\vert\\
    &\leq \frac{2\psi\sqrt{d}}{c}\sup_{\mathbf{Y}\in\mathbf{R}^{d\times q}}||F(\mathbf{Y}) - \widehat{F}(\mathbf{Y})||\\
    & + \frac{2\psi d}{c(T - q - s)}\left(\max_{i = 1,\cdots,d}\sup_{\mathbf{Y}\in\mathbf{R}^{d\times s}}\vert G_i\left(\mathbf{Y}\right) - \widehat{G}_i\left(\mathbf{Y}\right)\vert\right)\left(\sum_{i = 1}^d \sum_{t = s + q + 1}^T\vert\boldsymbol{\eta}_{t,i}\vert\right)\\
    & + \frac{2C\psi d^2s}{c(T-q-s)}\sup_{\mathbf{Y}\in\mathbf{R}^{d\times q}} || F(\mathbf{Y}) - \widehat{F}(\mathbf{Y})||\sum_{i = 1}^d\sum_{t = s + q + 1}^T\vert\boldsymbol{\eta}_{t,i}\vert,
\end{align*}
and 
\begin{align*}
    \mathbf{E}\left[\frac{1}{T -q - s}\sum_{i = 1}^d\sum_{t = s + q + 1}^T\vert\boldsymbol{\eta}_{t,i}\vert\right] = \sum_{i = 1}^d\mathbf{E}\left[\vert \boldsymbol{\eta}_{1,i}\vert\right] < \infty.
\end{align*}
According to Assumption 2, 
    \begin{align*}
        \frac{\psi\sqrt{d}}{T - q - s}\sum_{t = s + q + 1}^T ||\boldsymbol{\Delta}_t||\to_p 0,
    \end{align*}
    and the result is proven according to the continuity of $P(\cdot),$ and by setting $\psi \to \infty$.
\end{proof}


\section{Additional experimental results}
\label{section.additional_experimental}

\subsection{Datasets}
\label{section.data_set}
\subsubsection{Synthetic Data}
The synthetic data experiment is based on two heteroskedastic time-series models. First, we generate data from a GARCH(1,1) process of the form 
$$
\mathbf{x}_t = \boldsymbol{\sigma}_t\boldsymbol{\eta}_t,\quad\text{where}\quad \boldsymbol{\sigma}_t^2 = 15 + 0.4\mathbf{x}_{t-1}^2 + 0.5 \boldsymbol{\sigma}_{t-1}^2,
$$
Second, we generate data from a GARCH-in-mean(1,1) process:
$$
\mathbf{x}_t = 0.08\boldsymbol{\sigma}_t^2 + \boldsymbol{a}_t,\quad \boldsymbol{a}_t = \boldsymbol{\sigma}_t\boldsymbol{\eta}_t,\quad \boldsymbol{\sigma}^2_t = 0.05  + 0.08 \mathbf{x}_{t-1}^2+ 0.9 \boldsymbol{\sigma}_{t-1}^2.
$$
In both settings, the innovations $\boldsymbol{\eta}_t$ are drawn from a Student's $t$-distribution with 5 degrees of freedom, and the sample size is 7200.

For a GARCH(1,1) model, according to  \cite{tsay2005analysis},  by defining $\boldsymbol{\tau}_t = \mathbf{x}_t^2 - \boldsymbol{\sigma}^2_t,$ Then the dynamics of the squared process can be written as
$$
\mathbf{x}_t ^2 = 15 + 0.9\mathbf{x}_{t-1} ^2 +\boldsymbol{\tau}_t  - 0.5\boldsymbol{\tau}_{t-1}. 
$$
That is, an ARMA process in  $\mathbf{x}_t^2.$ The GARCH(1,1) model is widely used in financial time-series analysis because it captures volatility clustering, a phenomenon that large shocks tend to be followed by large shocks and small shocks by small shocks.

The GARCH-in-mean model further allows the conditional variance to directly affect the conditional mean. 
Specifically, the term $0.08\boldsymbol{\sigma}_t^2$ introduces a volatility-dependent mean component, reflecting the commonly observed risk--return trade-off in financial markets, where higher conditional volatility may be associated with higher expected returns.

\subsubsection{Real Life Data}

The real-life datasets include \textit{Exchange}, \textit{S\&P 500 Industrial}, and \textit{Electricity} for univariate forecasting tasks, and \textit{ETTh1}, \textit{ETTh2}, and \textit{Electricity} for multivariate forecasting tasks. 
A summary of dataset characteristics is provided in Table~\ref{table.data_set_info}. 

The \textit{Exchange} dataset consists of daily exchange rates of eight countries from 1990 to 2016 and is widely used as a benchmark for long-term time-series forecasting. 
The \textit{S\&P 500 Industrial} dataset is constructed from 15 representative stocks in the Industrials sector of the S\&P 500 companies, with Ticker as follows:  
\[
\{\text{BA}, \text{CAT}, \text{MMM}, \text{GE}, \text{RTX}, \text{HON}, \text{ITW}, \text{RSG}, 
\text{ODFL}, \text{URI}, \text{LMT}, \text{NOC}, \text{CMI}, \text{CARR}, \text{WM}\}.
\]
For this dataset, we transform raw daily closing prices into log-returns,
\[
\mathbf{r}_t = \log(\mathbf{x}_t) - \log(\mathbf{x}_{t-1}),
\]
where $\mathbf{x}_t$ denotes the vector of closing prices at trading day $t$. 
This transformation is commonly used in financial time-series analysis as it alleviates the non-stationarity of price levels and improves comparability across assets. 
All stock series are aligned by trading dates.

Both \textit{Exchange} and \textit{S\&P 500 Industrial} are financial datasets and are therefore expected to exhibit typical financial-market characteristics, including heteroskedasticity, volatility clustering, heavy tails, and abrupt regime changes. 
These properties make them challenging for forecasting models, especially under distribution shifts or periods of high market uncertainty.

The \textit{Electricity} dataset contains electricity consumption records from multiple clients. 
It is characterized by strong daily and weekly seasonal patterns, while also containing irregular fluctuations caused by holidays, weather conditions, or changes in consumption behavior. 
For multivariate forecasting, we additionally use the \textit{ETTh1} and \textit{ETTh2} datasets, which contain hourly electricity transformer measurements, including oil temperature and load-related variables. 
These datasets are useful for evaluating whether forecasting models can capture both temporal dependencies and cross-variable interactions in multivariate industrial sensor data.

Together, these datasets cover different types of real-world time series, including financial markets, electricity demand, and industrial sensor measurements. 
They provide complementary benchmarks with different levels of noise, seasonality, non-stationarity, and cross-variable dependence.

\begin{table}[h]
  \centering
  \begin{threeparttable}
  \caption{Overview of the datasets statistics}
  \label{table.data_set_info}
  \begin{tabular}{lllllll}
    \toprule
    Dataset   & Dimension     & Test    & Domain   &Freq. & Median Time Steps    \\
    \midrule
        ETTh1\tnote{a}         & 7 & 126 & $\mathbf{R}^+$ & H   &  17396  \\ 
    ETTh2\tnote{b}         & 7 & 126 & $\mathbf{R}^+$ & H   &  17396  \\ 
    S\&P 500 Industrial\tnote{c} &15 &72 & $\mathbf{R}^+$ &D & 1259\\
    Electricity\tnote{d}  & 370 & 2590 & $\mathbf{R}^+$ & H &  5833  \\

    Exchange\tnote{e}    & 8 & 40 & $\mathbf{R}^+$ &  D & 6071 \\
    \bottomrule
  \end{tabular}
    \begin{tablenotes}
    \item[a] \url{https://github.com/zhouhaoyi/ETDataset/tree/main} 
    \item[b] \url{https://github.com/zhouhaoyi/ETDataset/tree/main}
    \item[c] \url{https://www.kaggle.com/datasets/camnugent/sandp500/data} 
    \item[d] \url{https://archive.ics.uci.edu/dataset/321/electricityloaddiagrams20112014}
    \item[e] \url{https://github.com/laiguokun/multivariate-time-series-data}
    \end{tablenotes}

    \end{threeparttable}
\end{table}

\subsection{Training Details}
\label{section.training_details}
\textbf{Common setup.} 
When computing CRPS, MSIS, and $\mathrm{ACE}_{90}$, we set \texttt{num\_samples} to 100 in \textit{GluonTS}.

\textbf{Univariate time series.} 
For ProbRes, the conditional mean model uses the same backbone architectures as in the main paper, namely DLinear, PatchTST, and TimeMixer. The context length and prediction length follow the settings in \cite{NEURIPS2023_5a1a10c2}, and are demonstrated in Table \ref{table.hyper_parameter}. For the conditional volatility model, we adopt the techniques described in Remarks~\ref{remark.train_G} and~\ref{remark.predict_A}. We instantiate this component with a simple multilayer perceptron, implemented as \texttt{SimpleFeedForwardEstimator} in \textit{GluonTS}~\cite{JMLR:v21:19-820}. The context length is carefully selected based on the autocorrelation coefficients plot provided below (Figure \ref{figure.appendix_acf}), which is 
summarized in Table \ref{table.hyper_parameter}, 
and the prediction length is set to 1. 
For quantile-based methods, we use quantile levels $[0.1, 0.5, 0.9]$. 
For conformalized quantile regression, we fix $\alpha=0.1$ and use a calibration ratio of $0.2$ in all experiments.

\textbf{Multivariate time series:} 
In all experiments, the conditional volatility model employs the same architecture as the conditional mean model.  The context length is empirically determined based on the autocorrelation coefficients plot (Figure \ref{figure.appendix_acf}), with specific values detailed in Table \ref{table.hyper_parameter}.

\textbf{Synthetic Data:} For all synthetic data, the first 90\% of samples serve as training set and the last 10\% samples serve as test set. The burn-in period is chosen to be 30.

\textbf{Computational resources.}
All experiments are run on two servers equipped with NVIDIA A10 GPUs.

\begin{table}[h]
  \caption{Hyperparameters of the Conditional Mean and Volatility model}
  \label{table.hyper_parameter}
  \centering
  \begin{tabular}{llllll}
    \toprule
    & \multicolumn{2}{c}{Conditional Mean Model}  & \multicolumn{2}{c}{Conditional Volatility Model} \\ 
    \cmidrule(r){2-3} \cmidrule(l){4-5}
    Dataset  & Context Len.     & Predict Len.   & Context Len.  & Predict Len.  \\
    \midrule
    ETTh1     & 336 & 24 & 24   &  1\\ 
    ETTh2     & 336 & 24 & 24   &  1\\ 
    S\&P500 Industrial & 72  & 12   &30 &  1 \\
    Electricity   & 336 & 24  &48 & 1 \\
    Exchange   & 360 & 30  &100 &  1\\
    GARCH(1, 1) & 72 & 12  &12 &  1\\
    GARCH-in-mean(1,1)& 72 & 12  &12 &  1\\
    \bottomrule
  \end{tabular}
  
\end{table}

\begin{figure}[h]
    \centering
    \begin{subfigure}[b]{0.32\textwidth}
        \includegraphics[width=\linewidth, height = 0.8\linewidth]{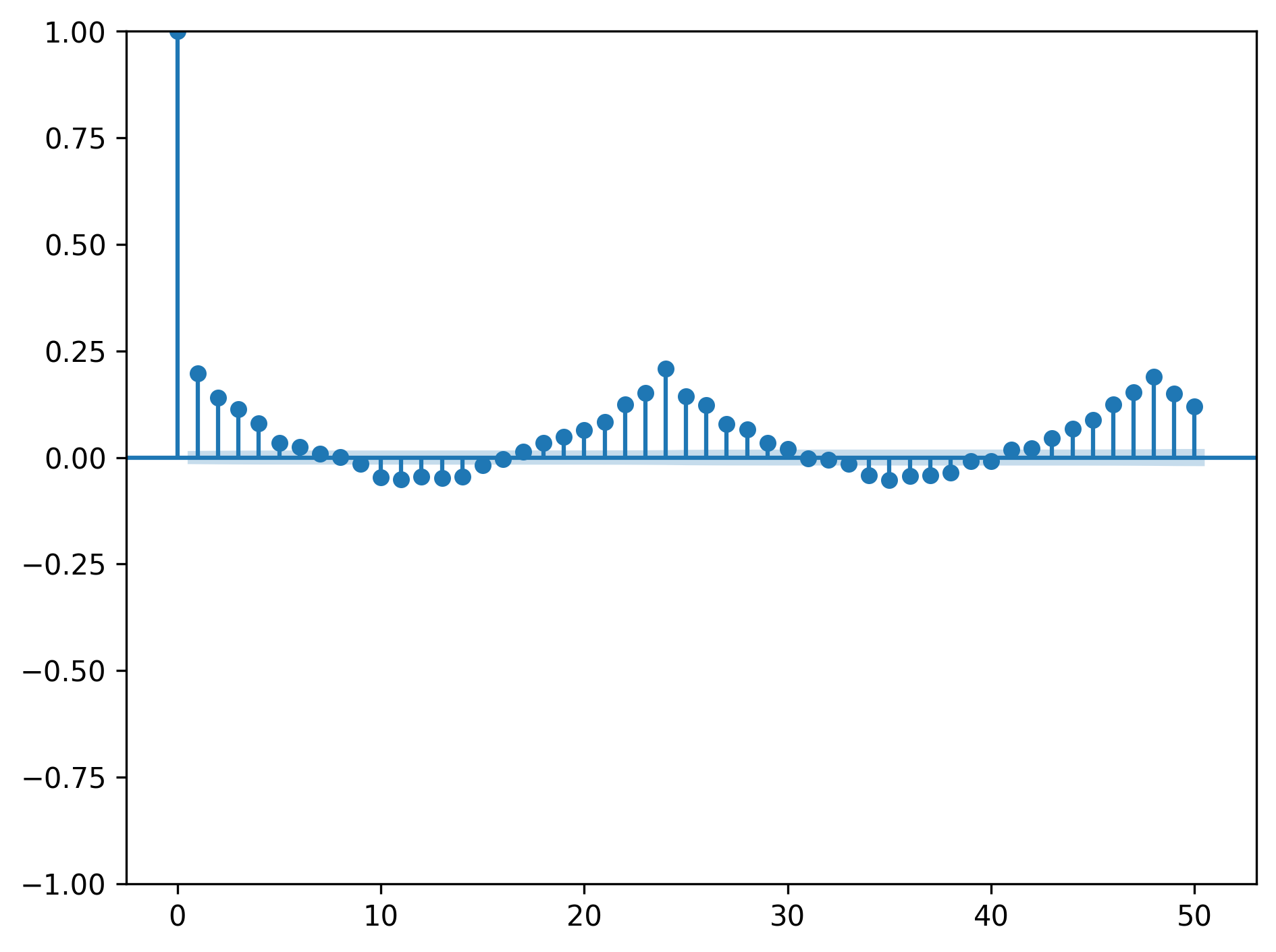}
        \caption{ETTh1}
    \end{subfigure}
    \begin{subfigure}[b]{0.32\textwidth}
        \includegraphics[width=\linewidth, height = 0.8\linewidth]{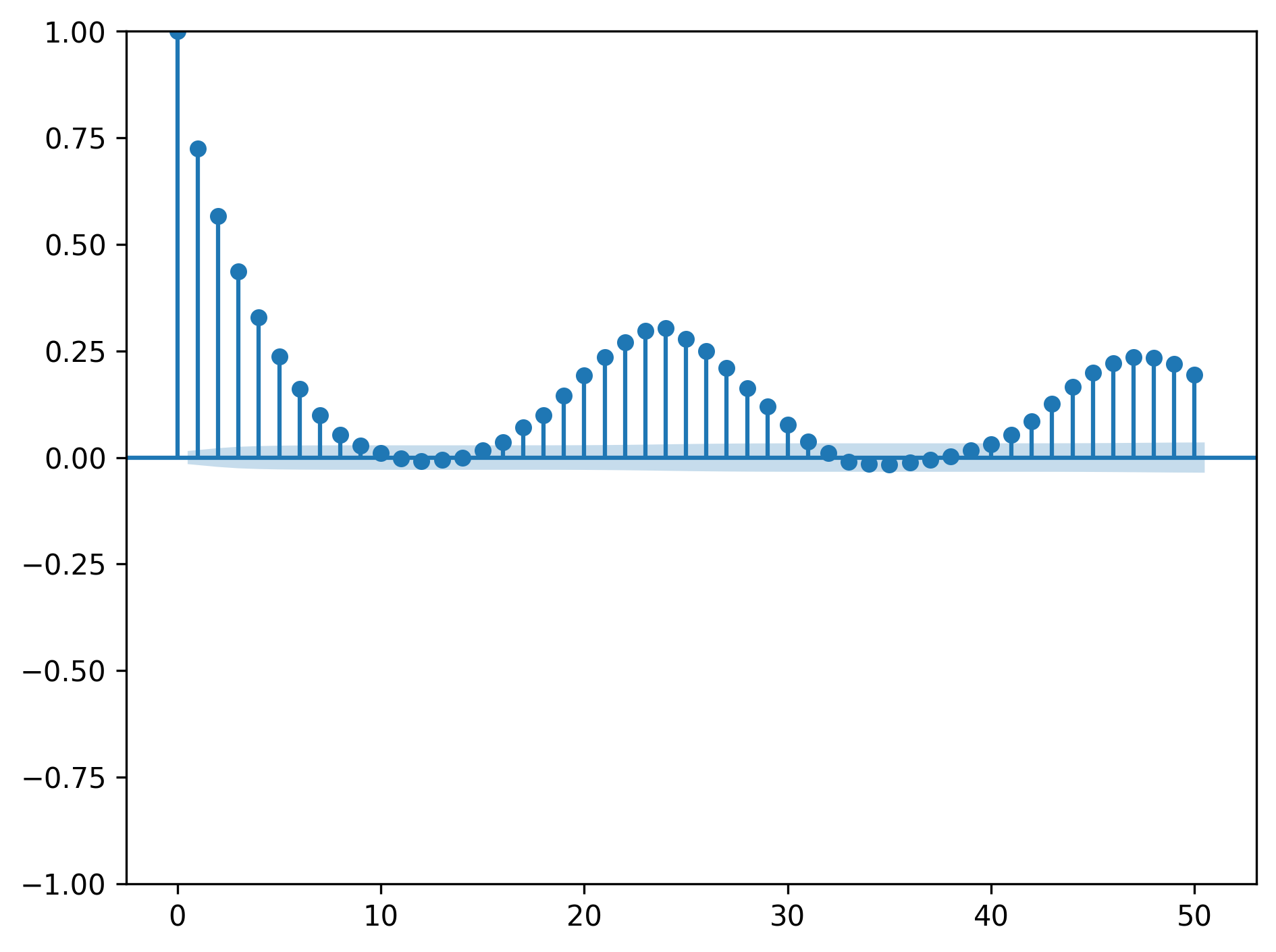}
        \caption{ETTh2}
    \end{subfigure}
    \begin{subfigure}[b]{0.32\textwidth}
        \includegraphics[width=\linewidth, height = 0.8\linewidth]{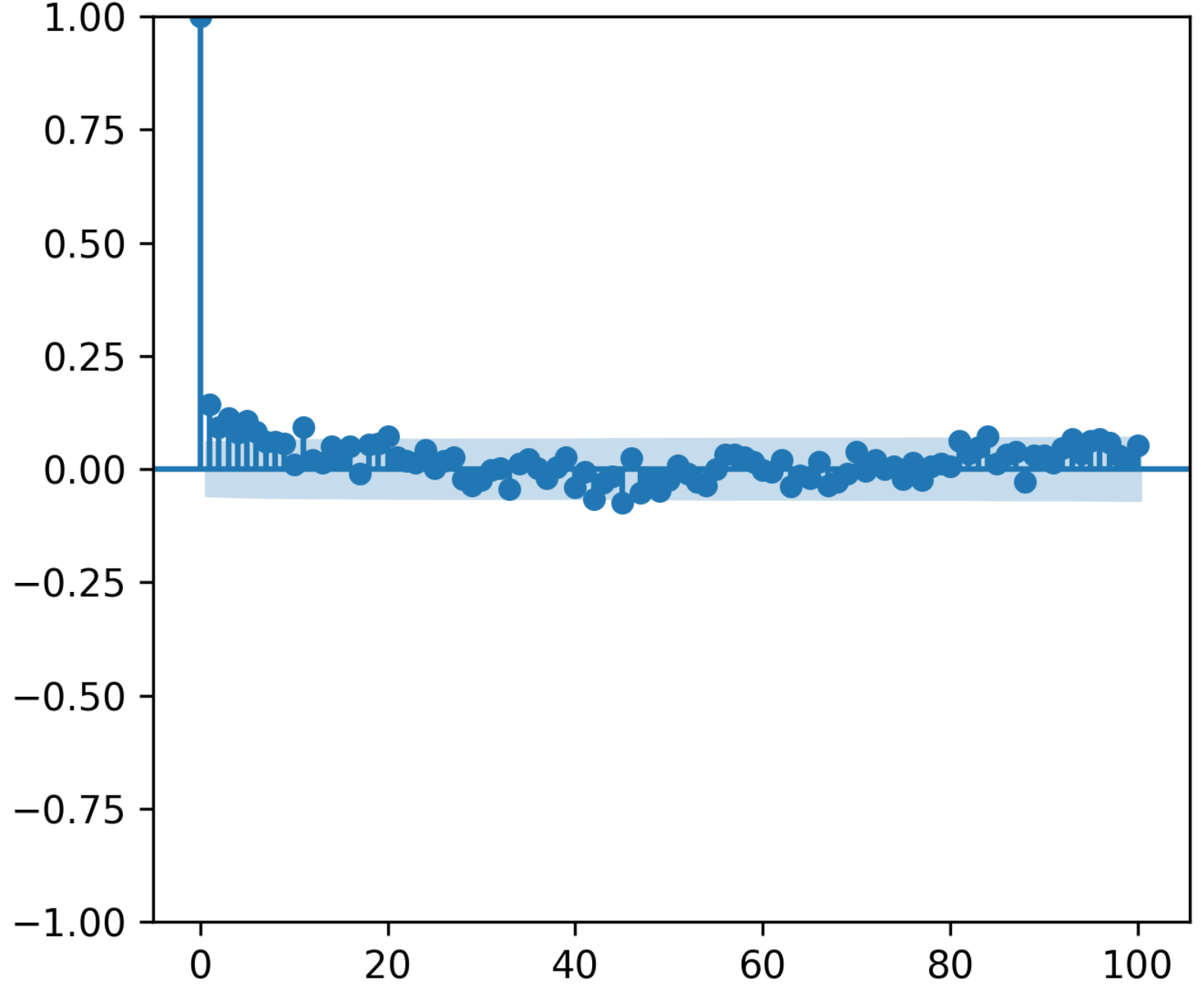}
        \caption{S\&P500 Industrial}
    \end{subfigure}
    \begin{subfigure}[b]{0.32\textwidth}
        \includegraphics[width=\linewidth, height = 0.8\linewidth]{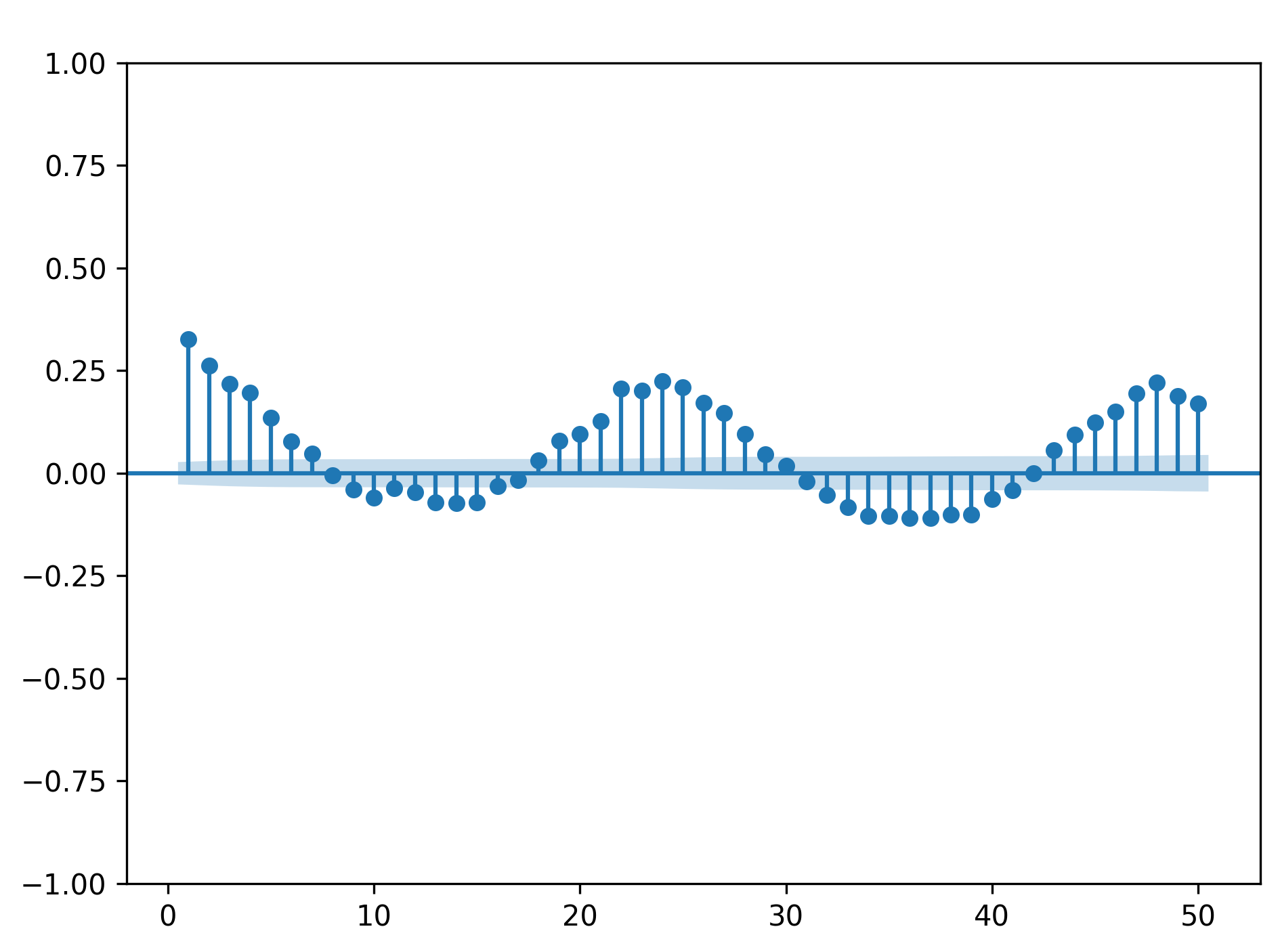}
        \caption{Electricity}
    \end{subfigure}
    \begin{subfigure}[b]{0.32\textwidth}
        \includegraphics[width=\linewidth, height = 0.8\linewidth]{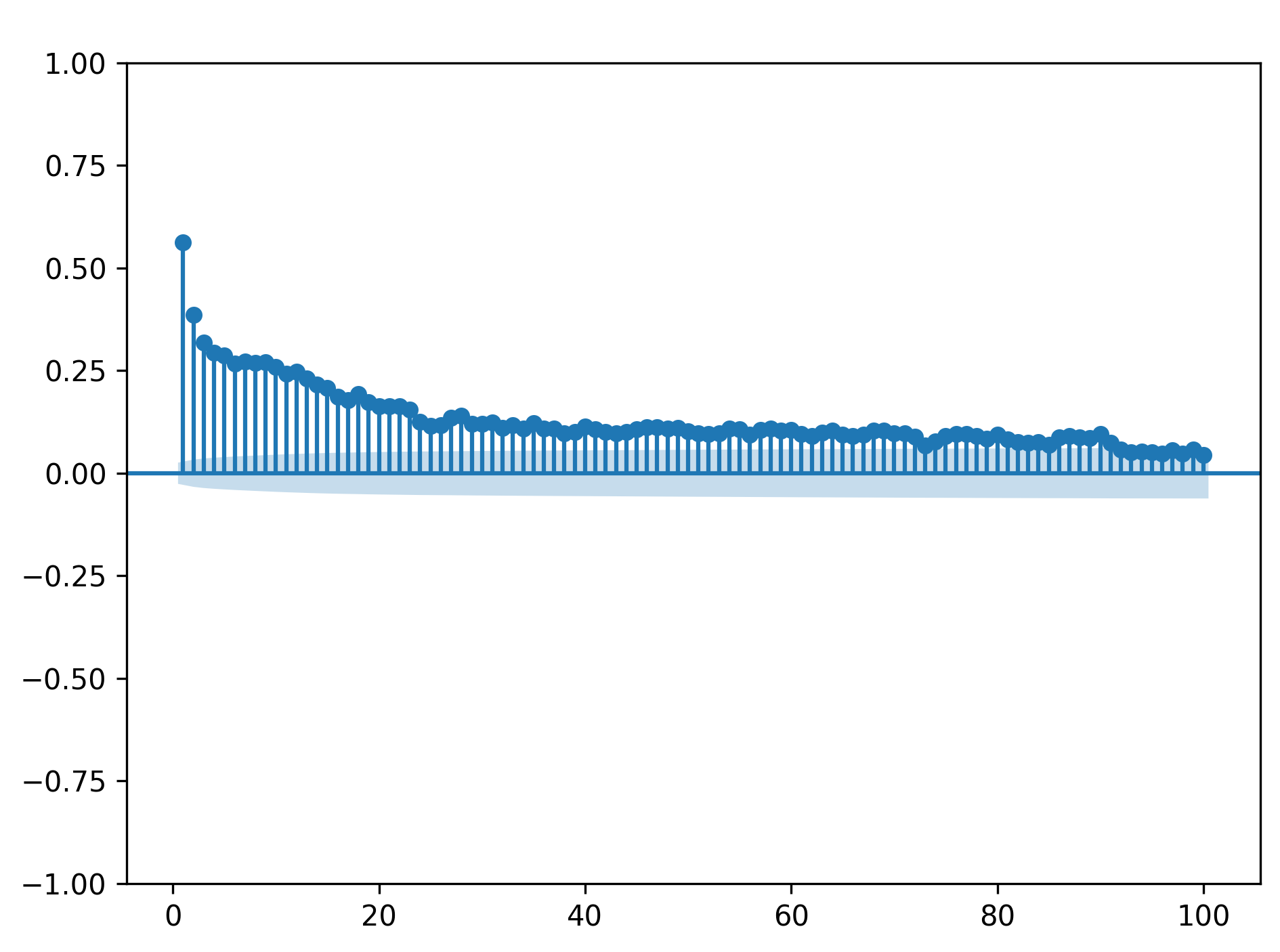}
        \caption{Exchange}
    \end{subfigure}
    \caption{Autocorrelation coefficients plot of the logarithm of square fitted residuals.}
    \label{figure.appendix_acf}
\end{figure}

\subsection{Metrics  of the experiment}
\label{section.metrics}
\textbf{Continuous Ranked Probability Score (CRPS).} The CRPS is a commonly used metric in probabilistic forecasting, as demonstrated in \cite{Gneiting01032007} and \cite{NEURIPS2023_5a1a10c2}. It is defined as the integral of the pinball loss over the interval $[0,1]$:
\begin{align*}
    CRPS(F^{-1}, y) = \int_0^1 2\Lambda_\kappa(F^{-1}(\kappa), y)\mathrm{d}\kappa,\ \text{where } \Lambda_\kappa(q,y) = (\kappa - \mathbf{1}_{y < q})\times (y - q).
\end{align*}
A forecasted quantile function $F^{-1}$ with a small CRPS indicates good alignment with the observation $y.$ We approximate the quantile function by sample quantiles at nine quantile levels $\{10\%, 20\%,\cdots, 90\%\}.$ These sample quantiles are estimated from 100 forecast samples.

For multivariate time series, the CRPS is computed as the summation of the element-wise CRPS.

\textbf{Mean Scaled Interval Score (MSIS). } Originally introduced in the M4 Competition \cite{MAKRIDAKIS2018802}, MSIS  simultaneously assesses the interval width and coverage in probabilistic forecasting. Suppose a prediction interval $[L_t, U_t]$ and the true value $y_t,$ the interval score (IS) at time $t$ is defined as 
$$
IS_t = \left(U_t - L_t\right) + \frac{2}{\alpha}(L_t - y_t)\mathbf{1}_{y_t < L_t} + \frac{2}{\alpha}(y_t - U_t)\mathbf{1}_{y_t > U_t},
$$
where $\alpha$ denotes the significance level. Failure in covering $y_t$ leads to an $2 / \alpha$ penalty in addition to the interval length. MSIS is defined by 
$$
MSIS = \frac{1}{J}\sum_{t = 1}^J \frac{IS_t }{\frac{1}{T - m}\sum_{t = m + 1}^T\vert y_t - y_{t - m}\vert},
$$
where $J$ is the forecast length, $T$ is the training sample size, and $m$ is the seasonality. We choose $m = 30$ and $\alpha = 5\%$ in the experiment.

\textbf{Average coverage error of the prediction interval ($\mathrm{ACE}_{90}$).} The metric $\mathrm{ACE}_{90}$ assesses  the calibration accuracy of prediction intervals.  Define PICP as in \cite{yao2019quality}, of the form 
$$
PICP = \frac{1}{J}\sum_{t = 1}^J \mathbf{1}_{y_t\in [L_t, U_t]}.
$$
ACE is defined as the deviation of PICP from the nominal coverage level  $ACE_{1 - \alpha} = \vert PICP - (1 - \alpha)\vert.$ We consider ACE at 90\% coverage, which selects $\alpha = 10\%.$

\textbf{Energy Score (ES).}  Introduced in \cite{chung2024samplingbased}, ES is a metric to evaluate the performance of a probabilistic forecasting method in capturing spatial dependence for multivariate data. For a future time series data  $\mathbf{y}_j\in\mathbf{R}^d,$ and a predictive distribution $\widehat{p}_j,$ we define the energy score as 
$$
ES_j = \mathbf{E}_{\mathbf{x}\sim\widehat{p}_j}||\mathbf{x} - \mathbf{y}_j||^\beta_2 - \frac{1}{2}\mathbf{E}_{\mathbf{x},\mathbf{x}^\prime\sim\widehat{p}_j}||\mathbf{x} - \mathbf{x}^\prime||^\beta_2,
$$
where $\mathbf{x},\mathbf{x}^\prime$ are independent sampled from $\widehat{p}_j.$ We calculate the ES as the average value 
$$
ES = \frac{1}{J}\sum_{j = 1}^J ES_j.
$$

Following \cite{chung2024samplingbased}, we set $\beta = 1.7.$ A smaller energy score indicates that the predictive distribution is closer to the ground truth.




\subsection{Additional experimental results}
\label{section.additional_result}

Tables~\ref{Table.experiment_result_real_life_data_appendix_1}, \ref{Table.experiment_result_real_life_data_appendix_2}, and \ref{Table.experiment_result_multivariate_appendix} complement Tables~\ref{Table.experiment_result_real_life_data_main} and
\ref{Table.experiment_result_multivariate} by reporting 95\% margins of error.
Overall, the results indicate that ProbRes achieves competitive performance compared to baselines with moderate variability, demonstrating stable behavior across datasets and backbone models.

On real-world datasets, ProbRes consistently improves CRPS and MSIS while maintaining well-calibrated coverage, as reflected by $\mathrm{ACE}_{90}$. In most cases, it either attains the best performance or remains highly competitive with the strongest baselines. Figure~\ref{fig:five_subplots} further illustrates the 95\% prediction intervals produced by ProbRes, which successfully cover the majority of future observations even in the presence of strong conditional heteroskedasticity (e.g., in the \textit{S\&P 500 Industrial} dataset). At the same time, the intervals remain reasonably sharp, highlighting their practical utility for downstream tasks such as asset pricing and risk management.

For multivariate datasets, Table~\ref{Table.experiment_result_multivariate_appendix} shows that incorporating ProbRes leads to consistent improvements in both CRPS$_{\text{sum}}$ and energy score (ES) across different models and datasets. The gains are particularly strong on the \textit{Electricity} dataset, suggesting that ProbRes is  effective in complex, high-variance settings.

Finally, we highlight an important distinction between datasets. The \textit{Electricity} dataset exhibits strong deterministic structure, with the mean component dominating the dynamics. In contrast, financial datasets such as \textit{S\&P 500 Industrial} log-returns are largely driven by stochastic fluctuations with limited mean structure. In such settings, accurately modeling volatility becomes essential for achieving reliable probabilistic forecasts, which further underscores the advantage of ProbRes.

\begin{table}[h]
  \caption{Experimental results on real-life data with 95\% margins of error.}
  \label{Table.experiment_result_real_life_data_appendix_1}
  \scriptsize
  \centering
  \begin{tabular}{l | c c c | c c c}
    \toprule
     Dataset &\multicolumn{3}{c|}{Electricity} &\multicolumn{3}{c|}{Exchange} \\
    \midrule
    Metrics & CRPS & MSIS & $\mathrm{ACE}_{90}$ & CRPS & MSIS & $\mathrm{ACE}_{90}$ \\
    \midrule
    Dlinear + q & 0.065(0.010) &8.380(0.013) &\underline{0.065(0.009)} & \underline{0.015(0.001)} & 38.02(5.854) & \underline{0.053(0.013)}\\ 
    + bootstrap &\underline{0.057(0.001)} & \underline{7.986(3.745)} &0.096(0.001) & 0.017(0.009) & 40.71(12.36) & 0.235(0.160)\\
    + CQR   &0.166(0.008) &17.18(0.703) &0.087(0.001) & \underline{0.015(0.001)} & \underline{30.50(3.051)} & 0.059(0.020)\\
    + ProbRes &\textbf{0.054(0.000)} & \textbf{7.602(3.790)} & \textbf{0.055(0.012)} & \textbf{0.010(0.001)} & \textbf{20.83(4.214)} & \textbf{0.048(0.047)}\\
    \midrule
    PatchTST + q & 0.072(0.007) & 6.758(0.904) & \underline{0.044(0.019)} &  \underline{0.013(0.003)} & \underline{38.65(9.148)} & 0.104(0.068)\\
     + bootstrap &\underline{0.064(0.000)} &\underline{6.599(0.233)} &0.081(0.011) & 0.026(0.018) & 101.5(91.02) & 0.365(0.162)\\
     + Conformal &0.106(0.004) &10.48(0.180) &0.086(0.002) & 0.016(0.003) & 42.03(15.28) & \underline{0.060(0.017)}\\
     + ProbRes &\textbf{0.063(0.001)} & \textbf{5.825(0.652)} & \textbf{0.041(0.015)}   & \textbf{0.012(0.001)} & \textbf{22.73(9.185)} & \textbf{0.035(0.016)}\\
     \midrule
      TimeMixer + q &0.254(0.035) &14.47(0.536) &0.064(0.008) & 0.029(0.002) & 77.00(4.588) & 0.056(0.026)\\
      + bootstrap &0.271(0.002) &\textbf{9.999(0.820)} &0.095(0.015) & \underline{0.016(0.001)} & \underline{63.22(2.602)} & \underline{0.026(0.004)}\\
      + Conformal & 0.258(0.022)&12.67(0.536) &\underline{0.030(0.022)} & 0.029(0.005) & 67.61(2.722) & 0.070(0.022)\\
      + ProbRes &\textbf{0.235(0.003)}&\underline{10.67(0.511)} &\textbf{0.007(0.004)}    & \textbf{0.013(0.001)} & \textbf{32.64(6.815)} & \textbf{0.015(0.021)}\\
    \bottomrule
  \end{tabular}

  \textbf{Remark: } q here refers to the quantile regression 
\end{table}

\begin{table*}[t]
  \caption{Continued Experimental results on real-life data with 95\% margins of error.}
  \label{Table.experiment_result_real_life_data_appendix_2}
  \setlength{\tabcolsep}{3pt}
  \small
  \centering
  
  \begin{tabularx}{\textwidth}{l |X X X }
    \toprule
     Dataset &\multicolumn{3}{c|}{S\&P 500 Industrial} \\
    \midrule
    Metrics & CRPS & MSIS & $\mathrm{ACE}_{90}$ \\
    \midrule
    Dlinear + quantile &\underline{0.817(0.003)} &5.494(0.066) &0.048(0.004)\\ 
    + bootstrap &0.826(0.000) &5.605(0.012) & \textbf{0.012(0.001)}\\
    + CQR      &0.833(0.003) &\underline{5.333(0.060)} &0.031(0.003)\\
    + ProbRes &\textbf{0.816(0.003)} &\textbf{5.222(0.097)} & \underline{0.030(0.008)}\\
    \midrule
    PatchTST + quantile & \underline{0.813(0.002)} & 5.217(0.102) & 0.044(0.016)\\
     + bootstrap & 0.821(0.001) & 5.537(0.022) & \textbf{0.013(0.002)}\\
     + Conformal & 0.821(0.002) & \underline{5.190(0.033)} & 0.032(0.002)\\
     + ProbRes   & \textbf{0.811(0.003)} & \textbf{5.170(0.066)} & \underline{0.024(0.011)}\\
     \midrule
      TimeMixer + quantile & \textbf{0.803(0.001)} & \textbf{4.966(0.029)} & 0.031(0.004)\\
      + bootstrap & 0.822(0.000)   & 5.653(0.002)        & \underline{0.018(0.000)}\\
      + Conformal & 0.813(0.001)   & \underline{5.012(0.014)}        & 0.043(0.002)\\
      + ProbRes & \textbf{0.803(0.002)}     & 5.073(0.087)        & \textbf{0.012(0.007)}\\
    \bottomrule
  \end{tabularx}
\end{table*}

\begin{table*}[h]
  \caption{Numerical experiment results on multivariate time series datasets. The interpretation of the values and the use of boldface are the same as in Table \ref{Table.experiment_result_synthetic_data}. }
  \setlength{\tabcolsep}{2pt}
  \small
  \label{Table.experiment_result_multivariate_appendix}
  \centering
  \begin{tabular}{l|cc|cc|cc}
    \toprule
    \multicolumn{1}{l|}{Dataset} & \multicolumn{2}{c}{ETTh1} & \multicolumn{2}{c}{ETTh2}  &  \multicolumn{2}{c}{Electricity}\\
    \midrule
    Metrics     & CRPS$_{sum}$  & ES   &CRPS$_{sum}$    & ES &CRPS$_{sum}$  & ES \\
    \midrule
   VEC-LSTM & 0.184(0.003) & 3.873(0.157) & 0.095(0.002) & 6.423(0.196) & 0.441(0.014)  & 48684(3323)\\
    +ProbRes   & \textbf{0.182(0.005)}  & \textbf{3.503(0.085)} & \textbf{0.087(0.001)} & \textbf{6.067(0.190)} & \textbf{0.301}(0.013) &  \textbf{41398}(3744)   \\
    \midrule
    TMDM  & 0.456(0.023)  & 13.344(0.163) & 0.092(0.008) & \textbf{6.933(0.393)} & 0.655(0.275)
      & 87761(6179)
    \\
    +ProbRes & \textbf{0.397(0.040)}  & \textbf{11.341(0.372)} & \textbf{0.092(0.004)}  & 7.326(0.498) & \textbf{0.292(0.018)}  & \textbf{37322(2438)}\\
    \bottomrule
  \end{tabular}

\end{table*}

\begin{figure}[t]
\centering

\begin{subfigure}{0.48\textwidth}
    \centering
    \includegraphics[width=\linewidth]{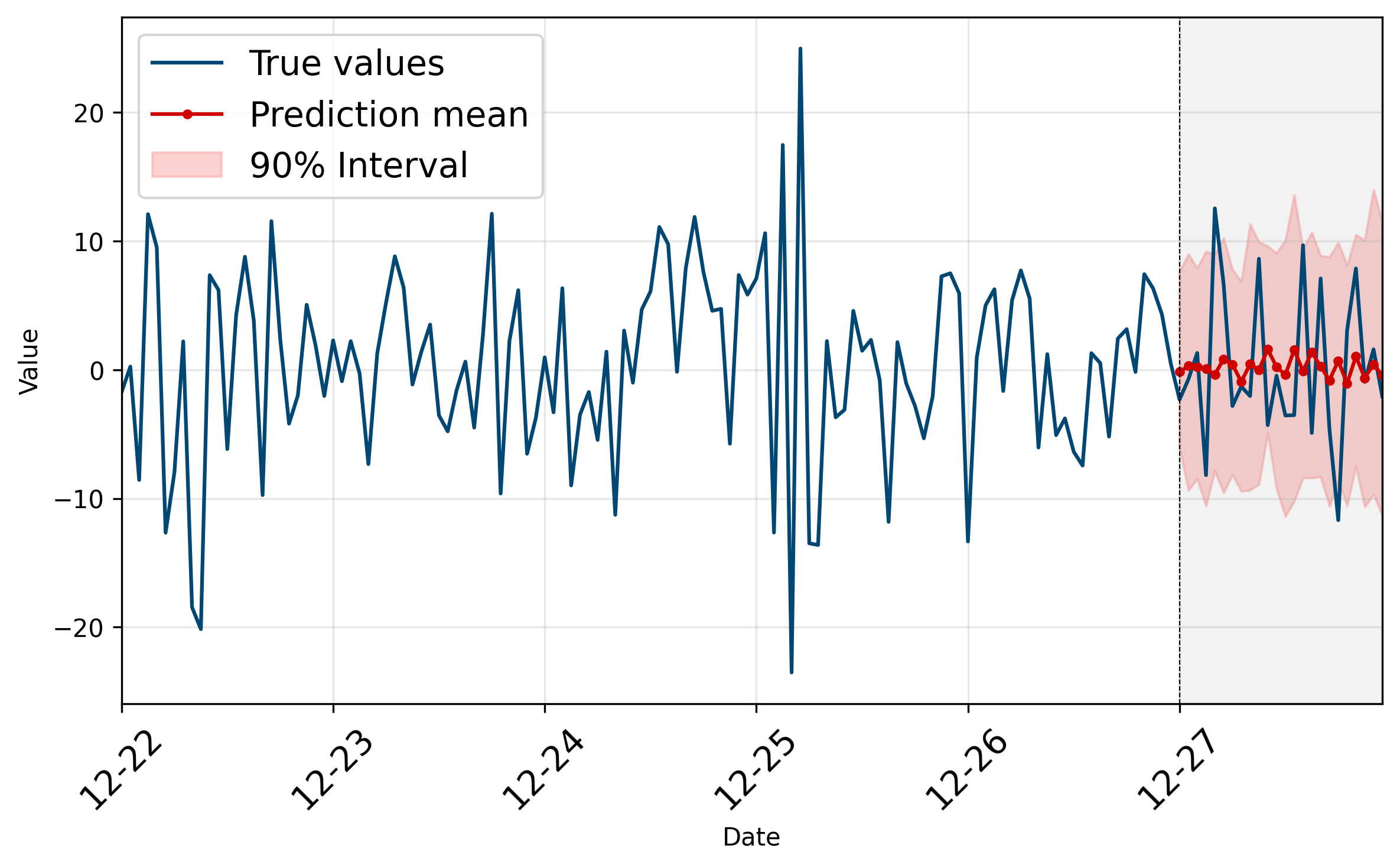}
    \caption{GARCH(1,1)}
\end{subfigure}
\hfill
\begin{subfigure}{0.48\textwidth}
    \centering
    \includegraphics[width=\linewidth]{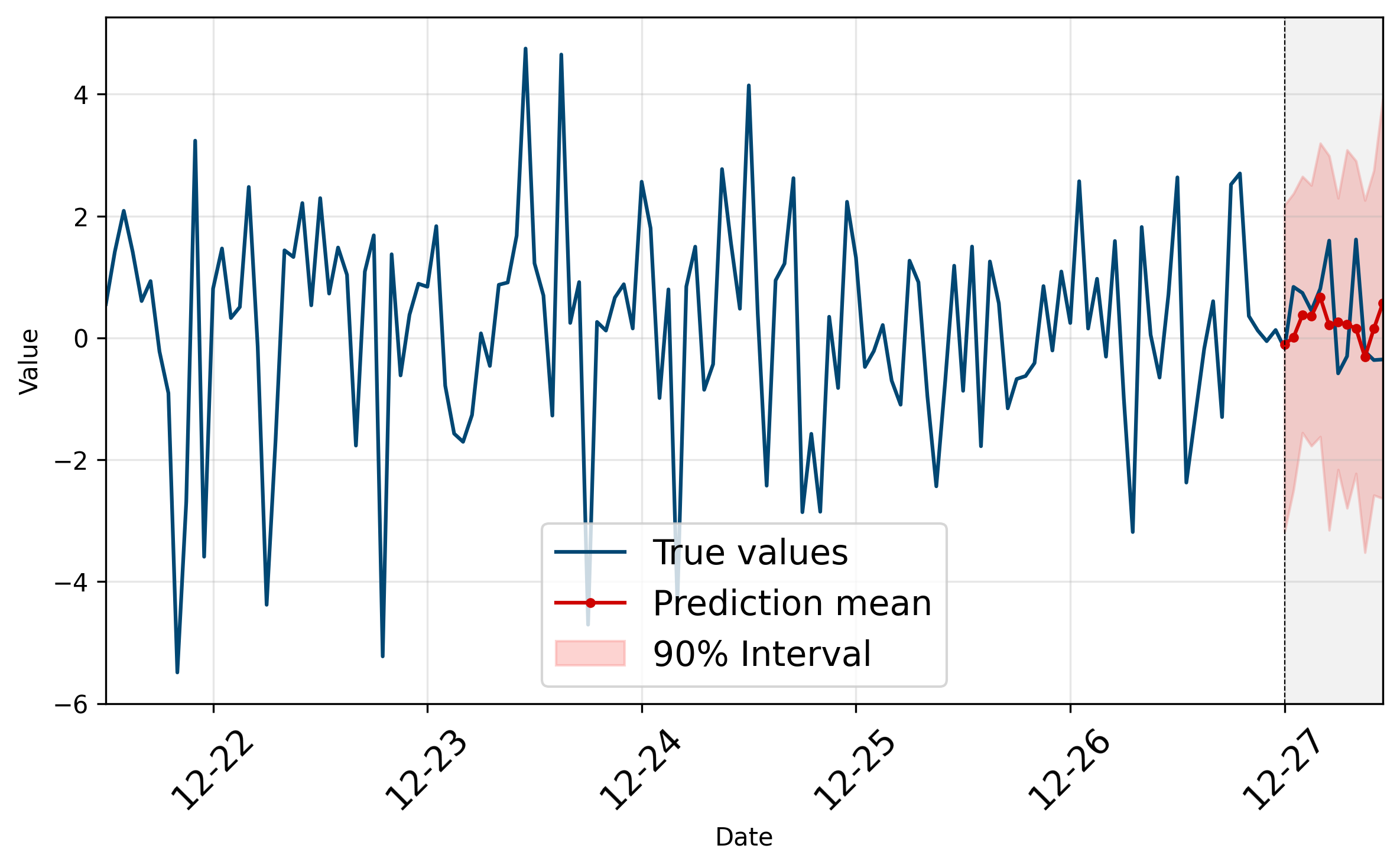}
    \caption{GARCH-in-mean(1,1)}
\end{subfigure}
\vspace{5pt}
\begin{subfigure}{0.48\textwidth}
    \centering
    \includegraphics[width=\linewidth]{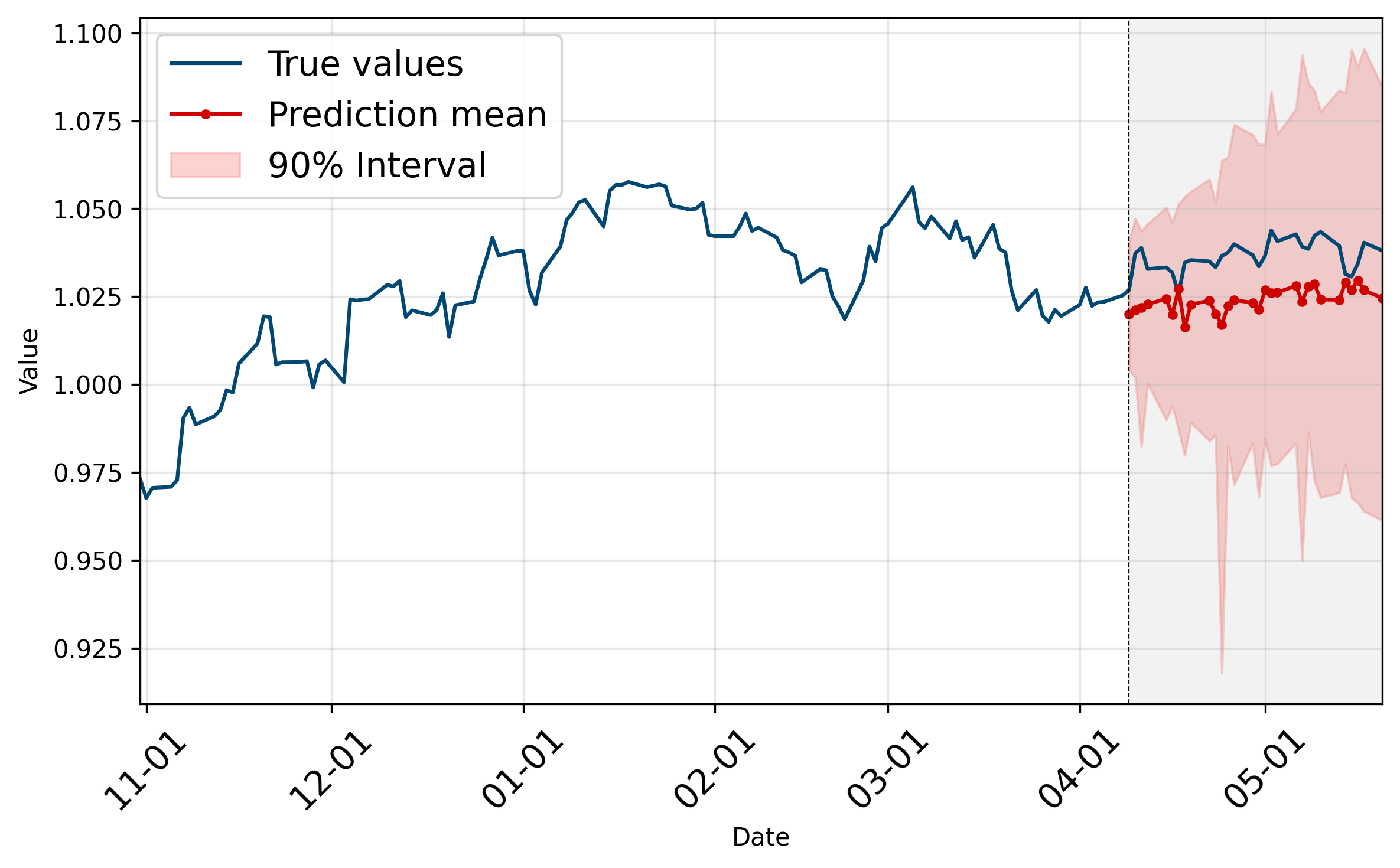}
    \caption{Exchange}
\end{subfigure}
\hfill
\begin{subfigure}{0.48\textwidth}
    \centering
    \includegraphics[width=\linewidth]{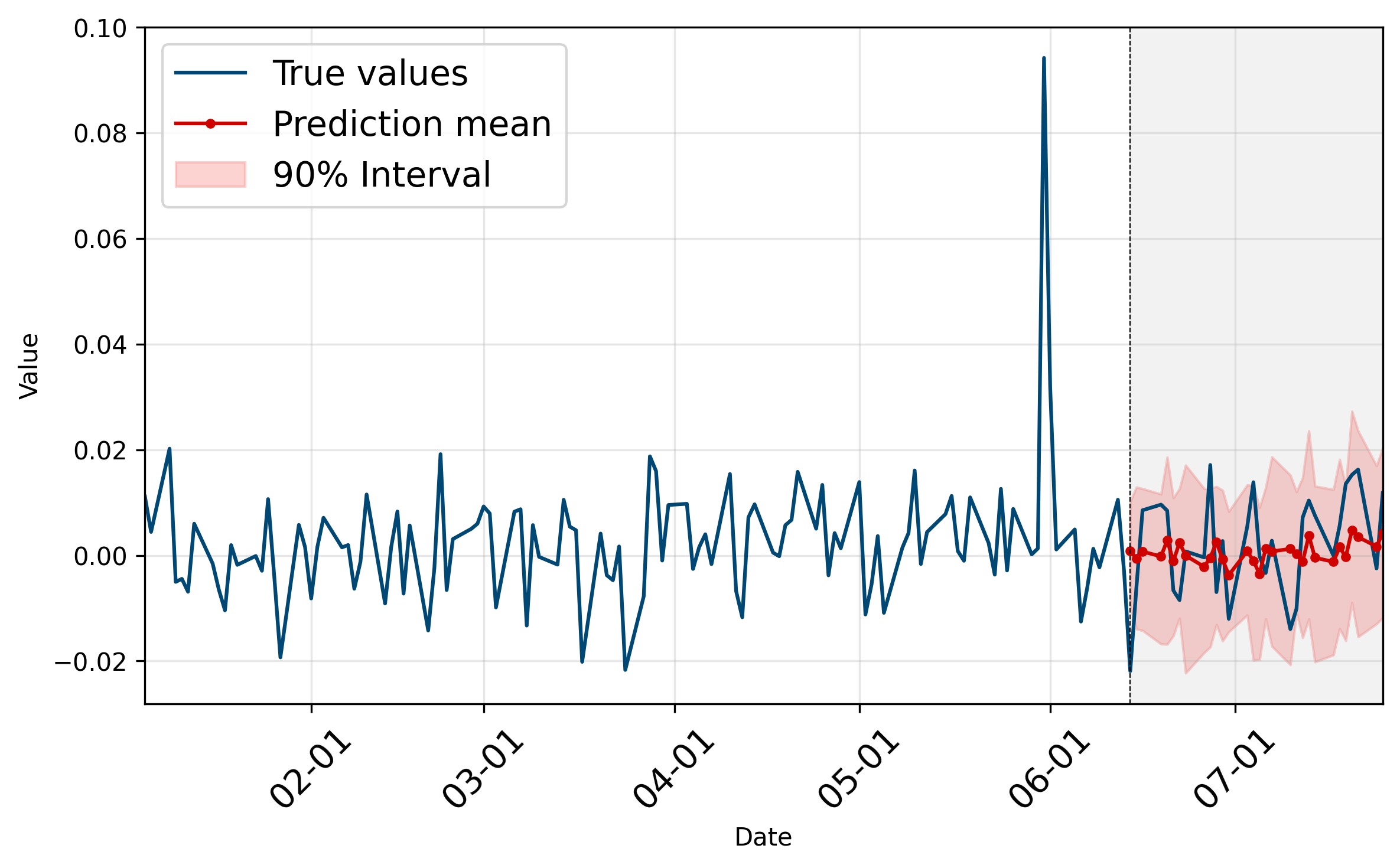}
    \caption{S\&P 500 Industrial}
\end{subfigure}
\vspace{5pt}
\begin{subfigure}{0.48\textwidth}
    \centering
    \includegraphics[width=\linewidth]{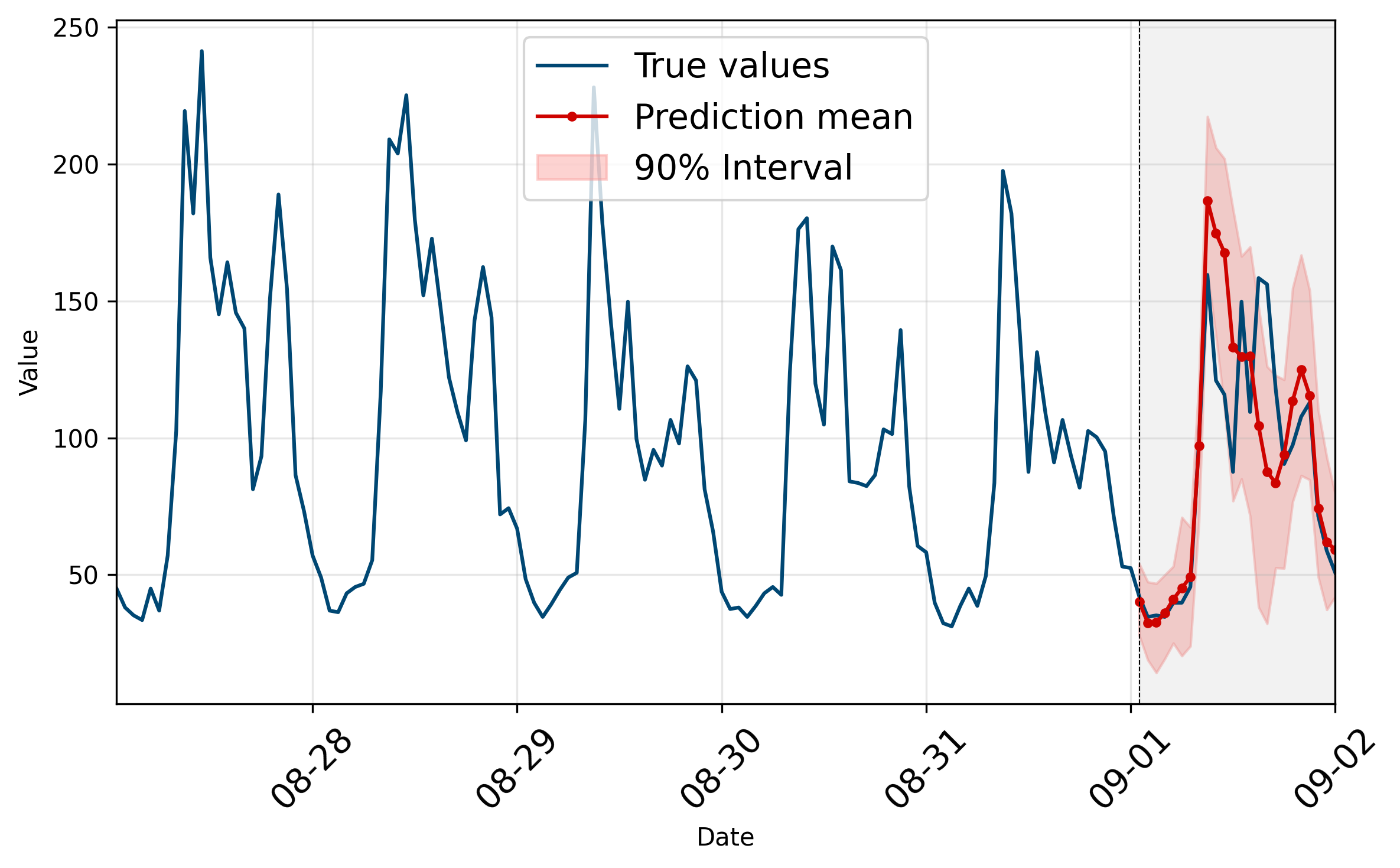}
    \caption{Electricity}
\end{subfigure}

\caption{90\% Prediction intervals generated by ProbRes with Dlinear backbone architecture.}
\label{fig:five_subplots}
\end{figure}

\textbf{Ablation study under distributional shifts:}   
This part evaluates the robustness of ProbRes to distributional shifts in future observations, focusing on both mean and scale perturbations. Specifically, we consider two types of shifts: 
\begin{itemize}
    \item A mean shift, implemented by adding a constant equal to 1\% of the dataset standard deviation to future observations.
    \item A scale shift, implemented by adding Gaussian noise with zero mean and standard deviation equal to 1\% of the dataset standard deviation to future observations.
\end{itemize}
These settings simulate mismatches in location and volatility between training and test distributions. The chosen magnitude (1\% of the standard deviation) corresponds to a moderate perturbation; for example, in the Electricity dataset (standard deviation $\approx$ 4000), this results in a shift of about 40 units.

The results in Table~\ref{Table.experiment_result_test_distribution_shift} show that ProbRes exhibits modest increases in CRPS under both types of shifts, indicating mild performance deterioration. This deterioration is generally stronger under mean shifts than under scale shifts. This suggests that ProbRes is relatively robust to volatility changes, while larger shifts in the mean of future observations have a stronger impact on predictive performance.

\begin{table}
\vspace{-5pt}
\centering
\caption{Changes in CRPS of ProbRes under different distributional shift scenarios. }
  \small
\begin{tabular}{l l|ccccc}
    \toprule
    Dataset & Backbone Model& None  & $\Delta$ Mean    & $\Delta$ Scale \\
    \midrule
    \multirow{3}{*}{Exchange } &   DLinear   & 0.009  & +0.0016 &  +0.0015\\
                               &    PatchTST & 0.012  & +0.0031 &  +0.0013\\
                               &    TimeMixer &0.014  & +0.0023 &  +0.0012\\
    \midrule
    \multirow{3}{*}{S\&P500 Industrial} &   DLinear & 0.814 & + 0.0011 & +0.0001\\
                               &    PatchTST        & 0.810 & + 0.0011 & $+6.3053\times 10^{-5}$\\
                               &    TimeMixer       & 0.809 & + 0.0010 & $+9.0884\times 10^{-6}$\\
    \midrule
    \multirow{3}{*}{Electricity} &   DLinear &  0.054 &  + 0.0288 & + 0.0266\\
                               &    PatchTST &  0.063 &  + 0.0287 & + 0.0250\\
                               &    TimeMixer & 0.235 &  + 0.0084 & + 0.0171\\
    \bottomrule
  \end{tabular}
 \label{Table.experiment_result_test_distribution_shift}
\vspace{-5pt}
\end{table}


\clearpage
\bibliographystyle{unsrt}
\bibliography{Arxiv_submission}

\end{document}